\documentclass[%
]{mpi2015-cscpreprint}

\usepackage[american]{babel}

\usepackage{subcaption}
\usepackage{graphicx}

\usepackage{float}

\usepackage{algorithm}
\usepackage{algpseudocode}

\usepackage{amssymb}
\usepackage{amsthm}
\newtheorem{theorem}{Theorem}
\newtheorem{remark}{Remark}
\usepackage{mathtools}
\usepackage[software,hardware]{mymacros}
\usepackage{amsmath}
\DeclareMathOperator*{\argmax}{arg\,max}

\DeclarePairedDelimiter\floor{\lfloor}{\rfloor}
\DeclareMathOperator{\Span}{span}

\def\*#1{\mathbf{#1}}

\newcommand{\gdrop}[1]{{\texttt{gDROP}}}
\newcommand{\drop}[1]{{\texttt{DROP}}}

\newcommand{\LSS}{{\texttt{LSS}}}
\newcommand{\bSv}{\mathbf{S}^{(\text{v})}}
\newcommand{\bSw}{\mathbf{S}^{(\text{w})}}

\newcommand{\bLambda}{\mathbf{\Lambda}}

\usepackage[normalem]{ulem}

\makeatletter
\DeclareOldFontCommand{\rm}{\normalfont\rmfamily}{\mathrm}
\DeclareOldFontCommand{\sf}{\normalfont\sffamily}{\mathsf}
\DeclareOldFontCommand{\tt}{\normalfont\ttfamily}{\mathtt}
\DeclareOldFontCommand{\bf}{\normalfont\bfseries}{\mathbf}
\DeclareOldFontCommand{\it}{\normalfont\itshape}{\mathit}
\DeclareOldFontCommand{\sl}{\normalfont\slshape}{\@nomath\sl}
\DeclareOldFontCommand{\sc}{\normalfont\scshape}{\@nomath\sc}
\makeatother

\usepackage[nolist,nohyperlinks]{acronym}

\renewcommand{\hat}[1]{\widehat{#1}}
\renewcommand{\tilde}[1]{\widetilde{#1}}
\begin{document}
\begin{acronym}[list]
\acro{drop}[DROP]{Dominant Reachable and Observable subspace-based Projection}
\acro{EVD}{eigenvalue decomposition}
\acro{LTI}{linear time-invariant}
\acro{lss}[LSS]{linear structured system}
\acroplural{lss}[LSS]{linear structured systems}
\acro{MIMO}{multiple-input-multiple-output}
\acro{MOR}{model order reduction}
\acro{PDE}{partial differential equation}
\acro{pod}[POD]{proper orthogonal decomposition}
\acro{rbm}[RBM]{reduced basis method}
\acro{ROM}{reduced-order model}
\acro{SISO}{single-input-single-output}
\acro{SVD}{singular value decomposition}
\end{acronym}


\title{Active Sampling of Interpolation Points to Identify Dominant Subspaces for Model Reduction}
  
\author[$1$]{Celine Reddig}
\affil[$1$]{Max Planck Institute for Dynamics of Complex Technical Systems, 39106 Magdeburg, Germany.\authorcr
  \email{creddig@mpi-magdeburg.mpg.de}, \orcid{0000-0001-9312-777X}}
  
\author[$2$]{Pawan Goyal}
\affil[$2$]{Max Planck Institute for Dynamics of Complex Technical Systems, 39106 Magdeburg, Germany.\authorcr
  \email{goyalp@mpi-magdeburg.mpg.de}, \orcid{0000-0003-3072-7780}}

\author[$3$]{Igor Pontes Duff}
\affil[$3$]{Max Planck Institute for Dynamics of Complex Technical Systems, 39106 Magdeburg, Germany.\authorcr
  \email{pontes@mpi-magdeburg.mpg.de}, \orcid{0000-0001-6433-6142}}

\author[$4$]{\\Peter Benner}
\affil[$4$]{Max Planck Institute for Dynamics of Complex Technical Systems, 39106 Magdeburg, Germany.\authorcr
  \email{benner@mpi-magdeburg.mpg.de}, \orcid{ 0000-0003-3362-4103}}
  
\shorttitle{Active sampling of interpolation points and model reduction}
\shortauthor{Reddig et al.}
\shortdate{}
  

  
\abstract{Model reduction is an active research field to construct low-dimensional surrogate models of high fidelity to accelerate engineering design cycles. In this work, we investigate model reduction for linear structured systems using dominant reachable and observable subspaces.  When the training set---containing all possible interpolation points---is large, then these subspaces can be determined by solving many large-scale linear systems. However, for high-fidelity models, this easily becomes computationally intractable. To circumvent this issue, in this work, we propose an active sampling strategy to sample only a few points from the given training set, which can allow us to estimate those subspaces accurately. 
To this end, we formulate the identification of the subspaces as the solution of the generalized Sylvester equations, guiding us to select the most relevant samples from the training set to achieve our goals.  
Consequently, we construct solutions of the matrix equations in low-rank forms, which encode subspace information. We extensively discuss computational aspects and efficient usage of the low-rank factors in the process of obtaining reduced-order models. We illustrate the proposed active sampling scheme to obtain reduced-order models via dominant reachable and observable subspaces and present its comparison with the method where all the points from the training set are taken into account. It is shown that the active sample strategy can provide us $17$x speed-up without sacrificing any noticeable accuracy. 
}

\novelty{
\begin{itemize}
    \item Active sampling for interpolation points to approximate reachable and observable subspaces 
    \item Recasting identification of subspaces as generalized Sylvester equations
    \item Solving these equations to obtain solutions as low-rank by utilizing the underlying structure of Sylvester equations in our framework
    \item Demonstrated the efficiency of the proposed active learning, showing computational speed-up of up to $17$x.
\end{itemize}}

\maketitle



\section{Introduction}\label{sec:intro}
 Mathematical models are a crucial component in the design cycles of engineering projects. With the increasing complexity of dynamic processes, these models can become intricate and high-fidelity, especially when governed by \acp{PDE}. This can make optimization, control, and simulation computationally expensive. To simplify these models and make them more efficient for design cycles, \ac{MOR} provides a solution. It allows for the construction of low-dimensional models by projecting high-fidelity models onto a low-dimensional subspace. These low-dimensional models accurately capture important dynamics and properties of the high-fidelity models, making the engineering design process faster and more efficient.

Numerous \ac{MOR} techniques have been proposed in the literature for linear and nonlinear systems from many perspectives; see, e.g. \cite{morBenMS05,morBenOCetal17, morBenSGetal21v1}. In this work, we focus on system-theoretic \ac{MOR}, in which balanced truncation \cite{morMoo81} and interpolation-based methods \cite{morGri97,morGalVV04, morAntBG20} are two widely known techniques.  One of the key advantages of system-theoretic \ac{MOR} methods is that they do not require extensive simulation to generate data for various configurations beforehand. This feature is particularly useful in applications where the dynamical model possesses control inputs, and the number of possible control configurations may be infinite, thus making data generation infeasible. Furthermore, in this work, we pay special attention to  linear structured systems, i.e. dynamical systems exhibiting particular structures possibly involving, e.g. second-order derivatives, time-delays, and integro-differential operators. \ac{MOR} for these classes of systems from system-theoretic perspectives has been widely investigated, see, e.g. for second-order systems \cite{morChaGVetal05,morReiS08,morSaaSW19} and for time-delay systems  \cite{morLam93,michiels2011krylov,morJarDM13}. Furthermore, balanced truncation \cite{morBre16}, and interpolation-based methods  \cite{morBeaG09} have been extended for general linear structured systems.  Additionally, a data-driven approach for learning linear structured systems from the frequency domain data has been proposed in \cite{morSchUBG18}.

In this work, our primary objective is to construct \acp{ROM} for high-fidelity linear structured systems by efficiently identifying their dominant reachable and observable subspaces as proposed in \cite{benner2019identification}. The authors in \cite{benner2019identification} connect interpolation-based \ac{MOR} methods with the reachability and observability dominant subspaces associated with linear structured systems. The method involves sampling frequency points in a desired range, followed by solving large-scale linear systems at all sampled points. This allows for a good approximation of reachable and observable subspaces. Then, based on the joint extraction of dominant reachable and observable information, low-dimensional models are constructed via a Petrov-Galerkin projection. As shown in \cite{benner2019identification}, this methodology provides very accurate \acp{ROM} compared to standard methods from the literature. However, the major bottleneck of this method is the requirement to solve large-scale linear systems at each of the sampled frequency points. Therefore, in order to bypass this issue, our approach in this work is to study the problem of selecting a few relevant frequency subsamples and constructing the dominant subspaces based on them, thus drastically reducing the number of large-scale linear solves required. 
 
The problem of selecting relevant subsamples has been addressed in different \ac{MOR} contexts. In the reduced basis method (RBM), one seeks to reduce the computational cost of solving parametric \acp{PDE}, e.g. \cite{morQuaMN16, morHaa17}. RBM relies  on constructing a small set of reduced basis functions that can approximate the solution of the PDE over a range of parameter values. The reduced basis is chosen by sampling the solution of the PDE at a few number of selected parameter values and using these solutions to form a \ac{ROM}. The choices of parameters are typically guided by error estimators and the use of a greedy algorithm; see, e.g.  \cite{morHaaO08a, morDroHO12, morHaa17}. Additionally, a typical choice of error estimator is based on the residual equation \cite{morHaaO08a}. In the context of interpolation-based \ac{MOR},  error estimators for the transfer function have been proposed in \cite{morFenAB17,morFenB19b}. Based on them, a few frequency subsamples are selected and used to construct \acp{ROM} via Galerkin (one-sided) projection only. Moreover, recent work has been dedicated to efficiently sample the parameter domain such that the training set is adaptively updated  \cite{morCheFdetal21a, morCheFB22}.  In the view of numerical linear algebra, some work has been dedicated to find the solutions to parameter-dependent matrix equations, e.g.  Lyapunov equations \cite{morSonS17,morPrzV21, lazar2021greedy} and Riccati equations \cite{morSchH15}, where error estimators have also been used to select parameter subsamples actively.  Additionally, greedy-inspired methods have also been discussed to solve the generalized Sylvester equation \cite{kressner2015truncated} using rank-1 updates. 

In this work, a new active sampling approach is proposed to sample only a few points from a dense frequency training set while approximating the reachable and observable subspaces up to the desired accuracy.  For this purpose,  we formulate the identification of reachable and observable subspaces as the solution of a large-scale generalized Sylvester equation with an interpolatory structure and a dense training set for the shift points is considered. Despite advances  in matrix equation solvers \cite{BenS13,Sim16}, solving generalized Sylvester equations remains a challenging task, especially when the number of terms is more than two. Additionally, matrix equation solvers also do not take into account the specific structure that arises in matrix equations from interpolatory \ac{MOR} methods. Hence, in this work, we take advantage of the interpolatory structure from the generalized Sylvester equation to compute low-rank solutions based on selected frequency subsamples.  To this end, we actively sample frequency points from a densely sampled training set based on the residual of the underlying generalized Sylvester equations. And we obtain matrices in low-rank forms, encoding reachable and observable subspace information. The obtained low-rank solutions are then used to determine the dominant subspaces, allowing us to efficiently compute accurate low-dimensional models. 

Furthermore, we have discussed several computational aspects to speed up the residual computation--which may be expensive--and the efficient use of the low-rank form in the construction of \acp{ROM}. We also discuss a scheme for choosing  multiple samples from the training set to call the residual computation less frequently.

The remainder of the paper is organized as follows. In \Cref{sec:Preliminaries}, we introduce the considered class of linear structured systems and  briefly recall interpolatory \ac{MOR} methods. Additionally, we outline the dominant subspace procedure from \cite{benner2019identification} and discuss a connection with the generalized Sylvester equation. Then, we present our active sampling strategy in \Cref{sec:active_sampling}, showing how to determine the most active points which are likely to contribute the most to provide relevant information about reachable/observable subspaces. In \Cref{sec:CompAspects}, we discuss several computational aspects related to the proposed active sampling strategy. Finally, \Cref{sec:Numerics} presents several numerical examples illustrating the performance of our active sampling strategy  and compares it to the method in \cite{benner2019identification}, where we observe a speed-up of up to $12$x without sacrificing any noticeable accuracy.  We conclude the paper with a brief summary and future research avenues.
\section{Preliminaries}\label{sec:Preliminaries}
In this section, we introduce a few concepts that will set the basis for the rest of the paper. We begin by presenting the form of linear structured systems that are considered in this paper. 
\subsection{Linear structured systems}
In this paper, we consider linear structured systems (denoted by \LSS), whose transfer functions---defining input-output mapping--- are of the form:
\begin{align}\label{eq:structured_sys}
    \bH\left(s\right) = \bC\mathcal{K}\left(s\right)^{-1}\bB,
    \end{align}
where  ${\mathcal{K}}(s) \in \C^{n\times n}$, and $\bB \in \R^{n\times m}$, and $\bC\in \C^{p\times n}$ are constant matrices, and $s$ takes values on the imaginary axis, i.e. $s\in i\R$; $n$ is referred as the dimension of the underlying linear model, which is often of order $10^4{-}10^6$; $m$ is the number of input controls and $p$ corresponds to the number of outputs. For a clear exposition, from now on we will consider the single-input single-output (SISO) case, i.e. when $m = p =1$. 
However, all the following discussions in the paper can be readily extended to multiple-input multiple-output (MIMO)  using the concept of tangential interpolation \cite{morGalVV04}.    Besides, we also assume that $\bH(s)$ is a strictly proper function, i.e.  $\lim\limits_{s\rightarrow \pm i\infty}{\bH(s)} = 0$. 
Furthermore, in this paper, we assume that $\cK(s)$ can be written in an affine form as follows:
\begin{align}\label{eq:K_affine}
     \mathcal{K}\left(s\right) = f_1\left(s\right)\bA_1 + \cdots + f_l\left(s\right)\bA_l
\end{align}
in which $\bA_i \in \mathbb{R}^{n \times n}$ and the functions $f_i\left(s\right): \C \rightarrow \C, ~i\in\{1,\ldots,l\}$ are meromorphic functions, which, under some mild conditions guarantees that $\cK\left(s\right)$ is invertible almost everywhere \cite{benner2019identification}. Systems \eqref{eq:structured_sys} satisfying assumption \eqref{eq:K_affine} cover several linear structured systems such as time-delay systems, second-order systems, or integro-differential systems. For example, a time-delay system with a single delay can be given as 
\begin{equation}
\begin{aligned}
    \bE\dot{\bx}(t) &= \bA\bx(t) + \bA_\tau\bx(t-\tau) + \bB\bu(t),\qquad \bx(0) = 0\\
    \by(t) &= \bC\bx(t)
\end{aligned}
\end{equation}
with transfer function $\bH(s) = \bC\left(s\bE - \bA - e^{-\tau s}\bA_\tau\right)^{-1}\bB$. When compared with the form given in \eqref{eq:K_affine}, we obtain
\begin{equation}
    \begin{aligned}
    f_1(s) &= s, & f_2(s) &= -1,& f_3(s) & = -e^{-\tau s},\\
    \bA_1 &= \bE, & \bA_2 &= \bA,& \bA_3 & = \bA_\tau.
    \end{aligned}
\end{equation}
\subsection{Interpolation-based MOR for \LSS}
Next, we briefly outline the framework of interpolation-based \ac{MOR}  for \LSS ~\eqref{eq:structured_sys}. The goal is to find \acp{ROM} of much smaller dimensionality while having the same structure as \eqref{eq:structured_sys} and accuracy. We aim to achieve this goal by using a Petrov-Galerkin projection of the high-fidelity system \eqref{eq:structured_sys}. More precisely, we seek to obtain a \ac{ROM} of order $r$ whose transfer function has the following form:

\begin{equation}
\label{eq:redsys}
    \hat{\bH}\left(s\right) = \hat{\bC}\hat{\mathcal{K}}\left(s\right)^{-1}\hat{\bB}, 
\end{equation}
where $\hat{\mathcal{K}}(s) \in \C^{r\times r}$, $\hat\bB \in \R^{r\times m}$, and $\hat\bC\in \R^{p\times r}$ are obtained by Petrov-Galerkin projection given by
\begin{equation}
\label{eq:redmat}
    \hat{\bC} = \bC\bV , \quad \hat{\mathcal{K}}\left(s\right) = \bW ^{\top}\mathcal{K}\left(s\right)\bV , ~~\text{and}~~ \hat{\bB} = \bW ^{\top} \bB, 
\end{equation}
with $\bV \in \R^{n \times r}$ and $\bW\in \R^{n \times r}$ being projection matrices. The choices of matrices $\bV$ and $\bW$ should be such the error between the orginal transfer function $\bH(s)$ and the reduced one $\hat{\bH}\left(s\right)$ is small enough, i.e. $\|\bH(s) -\hat{\bH}(s)\| \ll 1$ for all $s\in i\R$.

From the perspective of interpolation-based methods, we determine $\bV$ and $\bW$ such that the transfer functions of the high-fidelity model $\bH(s)$ and reduced model $\hat{\bH}(s)$ match at a set of given interpolation points. Such a problem for linear structured systems has been considered in \cite{morBeaG09}. In the following, we recall interpolation results (in simplified form) from  \cite{morBeaG09}.
\begin{theorem}[\cite{morBeaG09}]\label{thm:interpolation_beattie}
Let $\bH\left(s\right)$ be a transfer function as in \eqref{eq:structured_sys}. Consider a set of interpolation points $\left\{\sigma_1,\dots,\sigma_r\right\}$ and $\left\{\mu_1,\dots,\mu_r\right\}$ such that $\mathcal{K}\left(s\right)$ is invertible at these interpolation points. Furthermore, let $\bV$ and $\bW$ be defined as follows:
\begin{subequations}
\begin{align}
\underset{i=1,\dots,r}{\Span}
    \left\{\mathcal{K}\left(\sigma_i\right)^{-1}\bB\right\} &\subseteq  \range{\bV},\\
    \underset{i=1,\dots,r}{\Span} \left\{\mathcal{K}\left(\mu_i\right)^{-\top}\bC^{\top}\right\} &\subseteq  \range{\bW}.
\end{align}
\end{subequations}
If the reduced matrices are computed as in \eqref{eq:redmat} using these $\bV$ and $\bW$, then the following interpolation conditions are satisfied:
\begin{align*}
    \bH\left(\sigma_i\right) &= \hat{\bH}\left(\sigma_i\right),\\
    \bH\left(\mu_i\right) &= \hat{\bH}\left(\mu_i\right).
\end{align*}
Moreover, if $\sigma_i = \mu_i$ for $i = {1,\ldots, r}$, then the derivatives also match, i.e.
\begin{align*}
    \dfrac{d}{ds}\bH\left(\sigma_i\right) &= \dfrac{d}{ds}\hat{\bH}\left(\sigma_i\right).
\end{align*}
\end{theorem}
The previous theorem, therefore, enables us to construct a \ac{ROM} based on interpolation. Additionally, when $\cK(s)$ takes the affine form \eqref{eq:K_affine}, then the reduced $\hat\cK(s)$ also takes the same affine form:
\begin{align}\label{eq:K_affine_red}
     \hat{\cK}\left(s\right) = f_1\left(s\right)\hat\bA_1 + \cdots + f_l\left(s\right)\hat\bA_l
\end{align}
in which $\hat\bA_i = \bW^\top\bA_i\bV \in \mathbb{R}^{r \times r}$.

\begin{remark}\label{rem:close_conj} We assume that the high-fidelity realization \eqref{eq:structured_sys} is real. Therefore, in order to enforce that the obtained reduced model also has a real realization, the set of interpolation points $\cS = \left\{\sigma_1,\dots,\sigma_r\right\}$ is typically chosen to be closed under conjugation, i.e. if $\sigma \in \cS$ then $\overline{\sigma} \in \cS$, where $\overline{\sigma}$ denotes the complex conjugate of $\sigma$. In practice, while constructing the matrix $\bV$,  one needs to enforce that
\begin{align*}
\Span\;\left\{\mathcal{K}\left(\sigma\right)^{-1}\bB\right\} \subseteq \range{\bV} \quad \text{and} \quad \Span\;\left\{\mathcal{K}\left(\overline{\sigma}\right)^{-1}\bB\right\} \subseteq \range{\bV}
\end{align*}
so that the following interpolation conditions are satisfied
\[\bH\left(\sigma\right) = \hat{\bH}\left(\sigma\right) \quad \text{and} \quad  \bH\left(\overline{\sigma}\right) = \hat{\bH}\left(\overline{\sigma}\right).\] 
Henceforth, we assume that the set of interpolation points is always closed under conjugation.
\end{remark}

One of the main challenges in employing \Cref{thm:interpolation_beattie} is the choice of interpolation points, which determines the quality of the \acp{ROM}. For a particular class of linear systems,  
the choice of the interpolation points for a given order of \acp{ROM} can be based on the optimality conditions of the $\cH_2$-optimal \ac{MOR} problem \cite{morGugAB08}. Indeed, for linear systems, these optimality conditions are interpolatory conditions at specific points that are not known a priori. Therefore, those points are then learned through an iterative procedure. Since such methods are of iterative type, they may require many linear solves until they converge, provided they converge. Additionally, some extensions of the $\cH_2$-optimality conditions have been developed in the literature for \LSS~ with a specific structure, e.g. for second-order systems \cite{morBeaB14},  time-delay systems \cite{morPonGBetal16, morPonPS18},  and more general structured  systems \cite{morMli20}. 
However, their connections with interpolation-based projection methods still remain unclear.
Although choosing good interpolation points and order yet remains an open problem for \LSS, there is an alternative perspective from control theory, which connects interpolation-based subspaces with reachability and observability of \LSS. These ideas are presented in \cite{benner2019identification}, which are discussed in the following.


\subsection{Dominant Reachable and Observable Subspaces and Interpolation}\label{subsec:DROP}
Here, we present the core concepts from \cite{benner2019identification}, which connect interpolation-based \ac{MOR} with reachable/observable subspaces for \LSS~ followed by the identification of the dominant subspaces and by the truncation of the unreachable and unobservable subspaces. The main ingredient of the method is that the vectors---used to achieve interpolation; see \Cref{thm:interpolation_beattie}---also describe
reachable and observable subspaces. For a large enough number of interpolation points $\cN$, reachable and observable subspaces, respectively, can be encoded by the matrices $\bV$ and $\bW$ as follows:
\begin{subequations}\label{eq:intpol_main}
\begin{align}
    \label{eq:intpol}
     \bV & = \left[\mathcal{K}\left(\sigma_1\right)^{-1}\bB, \dots , \mathcal{K}\left(\sigma_\cN\right)^{-1}\bB\right],\\
     \label{eq:intpol2}
     \bW & = \left[\mathcal{K}\left(\sigma_1\right)^{-\top}\bC^{\top} , \dots , \mathcal{K}\left(\sigma_\cN\right)^{-\top}\bC^{\top}\right]
\end{align}
\end{subequations}
with
\begin{align*}
    \range{\bV} = \mathcal{V}_{\text{R}},\quad \text{and}\quad\range{\bW} = \mathcal{W}_{\text{O}},
\end{align*}
where $\cV_\text{R}$ and $\cW_{\text{O}}$, respectively, denote reachable and observable subspaces. Once we have these subspaces, we can determine dominant reachable and observable subspaces in a spirit similar to balanced truncation \cite{morMoo81} or the Loewner framework \cite{morMayA07,morAntGI16,morGosA18,morBenG21}. This is achieved by taking the \ac{SVD} of appropriate matrices. Then, \acp{ROM} are constructed using these dominant subspaces. The whole procedure is summarized in \Cref{alg:drop}, referred to as dominant reachable and observable subspace-based projection (\drop~).


%
Note that the number of interpolation points to describe reachable and observable subspaces has an upper bound, which is given by the order of the high-fidelity model. Computing at this many points using a high-fidelity model imposes a severe computational burden. However, we assume that there exists a low-dimensional model that can describe the dynamics of the high-fidelity system well. This implies that we can potentially construct reachable and observable subspaces with only a few interpolation points, which then raises the question of how to choose those few interpolation points. To address this, we aim to develop an active sampling strategy in \Cref{sec:active_sampling}  that allows us to achieve our goal, hence reducing the computational burden drastically. To do this, we first present a connection between the generalized Sylvester equation and the matrices $\bV$ and $\bW$ defined in \eqref{eq:intpol_main}, which is our key building block.

\begin{algorithm}[tb]
\caption{\textbf{D}ominant \textbf{R}eachable and \textbf{O}bservable subspace-based \textbf{P}rojection (\drop~) to Construct ROMs~\cite{benner2019identification} }\label{alg:drop}
 \hspace*{\algorithmicindent} \textbf{Input:} $\cK(s),~\bB$, and $\bC$, defining  transfer function as in \eqref{eq:structured_sys}, and \\ 
  \hspace*{\algorithmicindent}~~~~\qquad~ interpolation points $\{\sigma_1,\ldots,\sigma_\cN\}$ with $\cK(s)$ taking the affine form \eqref{eq:K_affine}.\\
    \hspace*{\algorithmicindent} \textbf{Output:} {The reduced-order matrices: $\mathcal{\hat{K}}\left(s\right)$, $\hat{\bB}$, $\hat{\bC}$.}
\begin{algorithmic}[1]
\State Compute $\bV$ and $\bW$ as defined in \eqref{eq:intpol} and \eqref{eq:intpol2} using the interpolation points $\sigma_i$. 
\State  Determine SVDs of the following matrices:
\begin{equation}\label{eq:drop_svd}
    \left[\bW^{\top}\bA_1\bV,\dots,\bW^{\top}\bA_l \bV\right] = \bW_1 \mathbf{\Sigma}_1\tilde{\bV}_1^{\top}\quad \text{and} \quad
    \begin{bmatrix}
        \bW^{\top}\bA_1\bV\\ \vdots\\ \bW^{\top}\bA_l \bV
    \end{bmatrix} = \tilde{\bW}_2\mathbf{\Sigma}_2\bV^{\top}_2
\end{equation}
\State Determine the order of the \ac{ROM} $r := \min(r_1,r_2)$, where $r_1$ and $r_2$ are so that the following conditions are satisfied:
\begin{equation*}
    \dfrac{\sum_{j=1}^{r_1}\lambda_1^{(j)}}{\sum_{j=1}^{\cN}\lambda_1^{(j)}} < \texttt{tol}\qquad \text{and}~~\dfrac{\sum_{j=1}^{r_2}\lambda_2^{(j)}}{\sum_{j=1}^{\cN}\lambda_2^{(j)}} < \texttt{tol},
\end{equation*}
with $\lambda_1^{(j)}$ and $\lambda_2^{(j)}$ being the $j$th diagonal entry of $\mathbf{\Sigma}_1$ and $\mathbf{\Sigma}_2$, respectively. 
\State Compute projection matrices: $\bV_{\text{p}} = \bV\bV_2\left(:,\;1{:}r\right)$ and $\bW_\text{p} = \bW\bW_1\left(:,\;1{:}r\right)$
\State Compute reduced matrices via Petrov-Galerkin projection:
\begin{align*}
    \hat{\bC} = \bC\bV_\text{p}, \quad \hat{\mathcal{K}}\left(s\right) = \bW_\text{p}^{\top}\mathcal{K}\left(s\right)\bV_\text{p}, \quad \hat{\bB} = \bW_\text{p}^{\top} \bB 
\end{align*}
\end{algorithmic}
\end{algorithm}

\subsection{Generalized Sylvester equations and computation of subspaces}\label{subsec:Sylv_Eq}
As noted in the previous subsection, the matrices $\bV$ and $\bW$, given in \eqref{eq:intpol} and \eqref{eq:intpol2}, need to be computed, which then allow us to encode dominant subspaces used for \ac{MOR}.  Each column of $\bV$ and $\bW$ requires solving a large-scale linear system for each chosen interpolation point. Next, we write the matrices $\bV$ and $\bW$ as solutions to generalized Sylvester equations. This is presented in the following theorem, which can be viewed as a generalization of the results from \cite{morGalVV04a} to the class of \LSS.

\begin{theorem} Let $\bV$ and $\bW$ be defined as in \eqref{eq:intpol} and \eqref{eq:intpol2}, respectively. Then $\bV$ and $\bW$ are also the solutions to the following generalized Sylvester equations:
    \begin{subequations}\label{eq:SLE}
    \begin{align}
        \bA_1 \bV \bF_1^{\bLambda} + \cdots + \bA_l \bV \bF_l^{\bLambda} &= \bB \mathbf{1}^{\top},\label{eq:SLE_v}\\
        \bA_1^\top \bW \bF_1^{\bLambda} + \cdots + \bA_l^\top \bW \bF_l^{\bLambda} &= \bC^\top \mathbf{1}^{\top},\label{eq:SLE_w}
    \end{align}
    \end{subequations}
    where  $\bLambda = \diag{\sigma_1,\dots, \sigma_\cN}$, $\bF_i^{\bLambda} = \diag{f_i(\sigma_1),\ldots, f_i(\sigma_\cN)}$ with $f_i(s)$ as in \eqref{eq:K_affine_red}, and $\mathbf{1}=\left(1,\dots, 1\right)^{\top} \in \mathbb{R}^{\cN \times 1}$. 
\end{theorem}

\begin{proof}
    First, note that for the $k$-th unit vector $\mathbf{e}_k$ we have 
    \begin{align}\label{eq:f_e_rel}
     \bF_i^{\bLambda}\mathbf e_k = f_i(\sigma_k) \mathbf e_k, \qquad i \in \{1,\ldots, l\}.
     \end{align}
Next, the $k$th column of $\bV$, denoted $\bv_k$, is given by
\begin{equation}
    \label{eq:3.3}
       \bV\mathbf e_k = \bv_k = \cK(\sigma_k)^{-1}\bB = \left(f_1(\sigma_k)\bA_1+\cdots+ f_l(\sigma_k)\bA_l\right)^{-1}\bB,
    \end{equation}
    which directly follows from \eqref{eq:intpol}. We multiply \eqref{eq:3.3} by $\left(f_1(\sigma_k)\bA_1+\cdots+ f_l(\sigma_k)\bA_l\right)$ from the left-hand side to obtain
    \begin{align}
       \left(f_1(\sigma_k)\bA_1+\cdots+ f_l(\sigma_k)\bA_l\right)\bV\mathbf e_k = \bB.
    \end{align}
    Rearranging the above equation and utilizing \eqref{eq:f_e_rel}, this results in
    \begin{align}
     \bB &=   f_1(\sigma_k)\bA_1\bV\mathbf e_k+\cdots+ f_l(\sigma_k)\bA_l\bV\mathbf e_k \nonumber\\
        &= \bA_1\bV \left(f_1(\sigma_k)\mathbf e_k\right)+\cdots+ \bA_l\bV \left(f_l(\sigma_k)\mathbf e_k\right) \nonumber\\
        &= \bA_1\bV\mathbf \bF_1^{\bLambda}\mathbf e_k+\cdots+ \bA_l\bV\mathbf \bF_l^{\bLambda}\mathbf e_k\label{eq:relation_v_e_b}
    \end{align}
    Since $\bB = \bB\mathbf{1}^{\top} \mathbf e_k$, we have the following relation:
    \begin{equation}
    \bA_1\bV\mathbf \bF_1^{\bLambda}\mathbf e_k+\cdots+ \bA_l\bV\mathbf \bF_l^{\bLambda}\mathbf e_k = \bB\mathbf{1}^{\top} \mathbf e_k,
    \end{equation}
    which holds for every $k \in \{1,\ldots, \cN\}$. This then yields
    \begin{equation}
    \bA_1\bV\mathbf \bF_1^{\bLambda}+\cdots+ \bA_l\bV\mathbf \bF_l^{\bLambda} = \bB\mathbf{1}^{\top}.
    \end{equation}
    Analogously, we can prove that $\bW$ also satisfies \eqref{eq:SLE_w}. This concludes the proof.
\end{proof}

The generalized Sylvester equations for $\bV$ and $\bW$ are used extensively to determine reachable and observable subspaces. In particular, they will provide a guidance tool to determine active samples for interpolation points that are likely to contain the most relevant information about subspaces. 

\section{Active  Sampling  of Interpolation Points}\label{sec:active_sampling}
As discussed in the previous section, if the number of interpolation points $\cN$ is large enough, the matrices $\bV$ and $\bW$ in \eqref{eq:intpol_main} encode the reachability and observability subspaces allowing the \drop~ procedure to construct accurate ROMs. However, if the set of interpolation points is too large, the construction of $\bV$ and $\bW$ may become computationally costly due to the large number of large-scale linear system solves it requires. In this section, we present the main contribution of this paper, namely, an active sampling procedure to learn  subsamples from a dense training of interpolation points. These selected interpolation points are then used to construct low-rank solutions of $\bV$ and $\bW$ via the Sylvester equations \eqref{eq:SLE_v} and \eqref{eq:SLE_w}.  As a consequence, we avoid the computation of linear solves for every point in the training set and, hence expect to significantly speed up the \drop~ procedure, which is outlined in \Cref{alg:drop}. 

More precisely, our goal is to construct the matrices $\bV$ and $\bW$ in \eqref{eq:intpol} and \eqref{eq:intpol2} without solving the corresponding linear equations at all interpolation points in the training set. 
We focus our discussion on the computation of $\bV$; however, the same arguments apply for computing the matrix $\bW$. As presented in \Cref{subsec:Sylv_Eq}, the matrix $\bV$ is the solution of the Sylvester equation \eqref{eq:SLE_v}. We assume that there exists a lower-dimensional model that can approximate the dynamics of the high-fidelity model. As a result, one can expect the reachability subspace to be approximable by a subspace of dimension much less than $n$. 
Thus, we seek to approximate the matrix $\mathbf{V} \in \R^{n\times \cN}$ in a low-rank form, i.e.
\begin{equation}\label{eq:V_aprrox_VZ}
    \bV \approx {\bSv}\bZ,
\end{equation}
where ${\bSv}\in \mathbb{R}^{n\times n_p}$ and $\bZ \in \mathbb{R}^{n_p\times \cN}$ with $n_p \ll \{n,\cN\}$. This also means that the solution of the Sylvester equation \eqref{eq:SLE_v} exhibits a low-rank structure.

To this end, we want to determine the matrix $\bSv$ using only a few samples from the training set of interpolation points, specifically by using only $n_p$ points from $\cN_d$, which is the size of the training set.
These interpolation points are chosen by an active search until the solution of the corresponding Sylvester equation in a low-rank form is  accurate up-to the desired tolerance. 
To measure the accuracy of the low-rank solution, we leverage the residual of the Sylvester equation as an error estimator. Employing the concept of greedy sampling, we select the point from the training set yielding the maximum residual as next sampling point. The selected interpolation point is expected to provide us with more relevant information about reachability than others.  

In the following, we mathematically formalize our active sampling scheme. To that end, let us denote a training set of interpolation points by $\{\sigma_1, \dots, \sigma_\cN\}$. 
We begin by constructing the matrix $\bSv_1$ by selecting an interpolation point $\sigma_1$ as follows:
\[\bSv_1 =  \text{orth}\left(\begin{bmatrix} \mathcal{K}(\sigma_{1})^{-1}\bB & \mathcal{K}(\overline{\sigma}_{1})^{-1}\bB\end{bmatrix}\right).  \]
Note that, due to \Cref{rem:close_conj}, we include both $\sigma_1$ and $\overline{\sigma}_1$ in order to enforce that the reduced system has a real realization.
The matrix $\bSv_1$ can then be used to approximate the matrix $\bV$ in the first iteration, i.e. $\bV_1 = \bSv_1\bZ_1$, where $\bZ_1$ is a matrix of appropriate size, and $\bV_1$ is an approximation of $\bV$ at the first step.  In general, we denote an approximation of $\bV$ at the $k$th step by $\bV_k$, which can be given by $\bV_k = \bSv_k\bZ_k$ with $\bSv_k$ and $\bZ_k$ being matrices of compatible size. 

Next, we assume that we are in the $k$th iteration and have the matrix $\bSv_k$, which is orthonormal, i.e. $\left(\bSv_k\right)^{\top}\bSv_k = \bI$.  As a first step, we need to determine $\bZ_k$ so that we can examine the quality of $\bV_k=\bSv_k\bZ_k$. 
For this, we leverage \eqref{eq:SLE_v}, where $\bV$ is replaced by its approximation $\bV_k = \bSv_k\bZ_k$, to get the following:
    \begin{equation}\label{eq:syl_approx}
    {\bA_1 \bSv_k\bZ_k \bF_1^{\bLambda} + \dots + \bA_l \left(\bSv_k\right)\bZ_k \bF_l^{\bLambda}} \approx \bB \mathbf{1}^{\top}  .
    \end{equation}
Projecting \eqref{eq:syl_approx} onto a low dimensional subspace by multiplying with $\left(\bSv_k\right)^{\top}$ from the left-hand side yields
\begin{align} \label{eq:proj}
    \left(\bSv_k\right)^{\top}\bA_1 \left(\bSv_k\right)\bZ_k \bF_1^{\bLambda} + \dots + \left(\bSv_k\right)^{\top}\bA_l \left(\bSv_k\right)\bZ_k \bF_l^{\bLambda} \approx \left(\bSv_k\right)^{\top}\bB \mathbf{1}^{\top}. 
    \end{align}
As a result, we can determine $\bZ_k$ by solving the following minimization problem:
\begin{align} \label{eq:sol_zk}
   \bZ_k = \arg\min_{\bZ_k}\left\| \left(\bSv_k\right)^{\top}\bA_1 \bSv_k\bZ_k \bF_1^{\bLambda} + \dots + \left(\bSv_k\right)^{\top}\bA_l \bSv_k\bZ_k \bF_l^{\bLambda} - \left(\bSv_k\right)^{\top}\bB \mathbf{1}^{\top}\right\|_F. 
    \end{align} 
Using $\bSv_k$ and $\bZ_k$, we can approximate $\bV \approx \bV_k = \bSv_k\bZ_k$, allowing us to determine the residual $\bR_k$ to estimate the quality of the approximation for $\bV$:
\begin{align}  \label{eq:residual}
    \bR_k = {\bA_1 \left(\bSv_k\right)\bZ_k \bF_1^{\bLambda} + \dots +\bA_l \left(\bSv_k\right)\bZ_k \bF_l^{\bLambda}} - \bB \mathbf{1}^{\top}. 
\end{align}
Since each row $j$ of $\bR_k$, denoted $\bR_k^j$, corresponds to the residual associated with the point $\sigma_j$ in the training set, we compute the residual norm column-wise as follows:
  \begin{align}
    \label{eq:norm}
        r_j^k = \left\| \bR_{k}^{j}\right\|_2, \quad \text{for $j = \{1, \dots, \cN\}$}. 
    \end{align}
Then, we introduce the following mean residual error, which is used to determine how well $\bV$ has been approximated by the current low-rank factorization:
    \begin{align}
        \label{eq:res-err}
        \epsilon_k= \dfrac{1}{\cN}\sum_{j=1}^{\cN}{r_j^k}.
    \end{align}
\newline
If the residual error $\epsilon_k$ is  below a certain tolerance \texttt{tol}, the procedure is stopped, and we have $\bV$ in a low-rank form of desired quality. Otherwise, a new interpolation point needs to  be determined to enrich the subspace of $\bSv_k$.  In the simplest way possible, a new point $\sigma_{\texttt{sel}}$ is chosen by 
    \begin{align}
    \label{eq: sigma}
        \sigma_{\texttt{sel}} = \sigma_{\texttt{idx}}, \quad \text{where} \quad \texttt{idx}= \argmax_{j=1,~\dots,~\cN}{~r_j^k}.
    \end{align}

In other words, we select the new interpolation point that corresponds to the largest error estimator (or residual). 
Once the adaptive procedure has stopped, we have $\bV \approx \bSv_k\bZ_k$ as an approximate solution to the Sylvester equation \eqref{eq:SLE_v}. As a consequence, $\bSv_k\bZ_k$ encodes the reachability subspace of the associated linear structured system and can be used in the \drop~ procedure in \Cref{alg:drop}. We sketch the pseudocode of the proposed active sampling procedure in \Cref{alg:greedy}. We remind that one can perform the same approach to obtain an approximation for $\bW$. Moreover, one can reuse the previously selected interpolation points while determining a low-rank factor for $\bV$ and use them for the initial construction of $\bW$ so that we have a good estimate of the observability  subspace, or an approximate solution of \eqref{eq:SLE_w} in  low-rank form, to begin with. 

Despite the above efforts, evaluating the residual $\bR_k$ can become computationally expensive.
As a remedy, we shall discuss how to take advantage of the incremental structure of $\hat{\bV}$ and avoid repeating calculations in \Cref{subsec:residual}. Moreover, we shall discuss an approach to choose multiple interpolation points in \Cref{subsec:mulpoint} so that the computation of the residual needs to be performed less often.
\begin{algorithm}[!tb]
\caption{Active sampling of interpolation points from a training set and solving Sylvester equation \eqref{eq:SLE} in low-rank form.}\label{alg:greedy}
 \hspace*{\algorithmicindent} \textbf{Input:} {Coefficient matrices of the Sylvester equation \eqref{eq:SLE_v}, a set of training points $\{\sigma_1,..., \sigma_\cN\}$, and \texttt{tol}}.\\
    \hspace*{\algorithmicindent} \textbf{Output:} An approximated solution of the Sylvester equation in low-rank form, i.e. {$\bV \approx \bSv_k\bZ_k$}.
\begin{algorithmic}[1]
\State {Set $\bSv_{1} = \text{orth}\left(\begin{bmatrix}\mathcal{K}(\sigma_{1})^{-1}\bB,~~\mathcal{K}(\bar{\sigma}_{1})^{-1}\bB\end{bmatrix}\right)$.}
\State {Solve the projected minimization problem \eqref{eq:sol_zk} for $\bZ_1$.}
\State {Compute the residual $\bR_1$ as in \eqref{eq:residual}}, and the residual error $\epsilon_1$ as in \eqref{eq:res-err}.
\State {Set $k = 1$.}
\While {error $\epsilon_k>$ \texttt{tol}}
    \State {Select new interpolation point $\sigma_{\texttt{sel}}$ as in \eqref{eq: sigma}}.
  \State {Update  $\bSv_{k+1}$ as follows:
  \[\bSv_{k+1} = 
        \text{orth} \left(\begin{bmatrix}\bSv_k,~ \mathcal{K}(\sigma_{\texttt{sel}})^{-1}\bB,~ \mathcal{K}(\overline{\sigma}_{\texttt{sel}})^{-1}\bB\end{bmatrix}\right).\]}
\State {Solve the projected minimization problem for $\bZ_{k+1}$ in \eqref{eq:proj}}.
\State {Compute the residual $\bR_{k+1}$ as in \eqref{eq:residual}}, and  the residual error $\epsilon_{k+1}$ as in \eqref{eq:res-err}.
\State Update $k \leftarrow k+1$.
\EndWhile
\end{algorithmic}
\end{algorithm}

\section{Computational Aspects}\label{sec:CompAspects}
In this section, we describe some crucial points allowing to speed up the proposed active sampling approach in \Cref{alg:greedy} and its combination with the \drop~ procedure in \Cref{alg:drop}.
\subsection{Solving the projected minimization problem}
\label{subsec:numex}
 An important part of \Cref{alg:greedy} consists of solving the minimization problem \eqref{eq:sol_zk} for $\bZ_k$ (see Steps $2$ and $8$ in \Cref{alg:greedy}).  To solve this equation, we take advantage of the special structure, especially the fact that $\mathbf F_i^\Lambda$ are diagonal matrices, and their diagonal entries are given by $ f_i(\sigma_i)$. 
 Then, under the assumption that 
 $${f_1(s)\left(\bSv_k\right)^\top\bA_1\bSv_k + \cdots + f_l(s)\left(\bSv_k\right)^\top\bA_l\bSv_k}$$
 is invertible for all $s \in \{\sigma_1,\ldots, \sigma_\cN\}$, it can be shown that the $j$th column of $\bZ_k$, denoted $\bz_k^j$, can be computed as follows:
     \begin{align*}
     \bz_{k}^j =  \left({f_1(\sigma_j)\hat\bA_1 + \cdots + f_l(\sigma_j)\hat\bA_l}\right)^{-1} \hat\bB,
    \end{align*}
where $\hat\bA_i = \left(\bSv_k\right)^\top\bA_i\bSv_k$ and $\hat\bB = \left(\bSv_k\right)^\top\bB$.

As a consequence, the solution of \eqref{eq:sol_zk} can be determined by solving projected linear systems. Note that these are low-dimensional linear systems and thus can be solved with low computational resources.

\subsection{Residual computation}\label{subsec:residual}

For the residual computation, we need to evaluate \eqref{eq:residual} at Steps 3 and 9 in \Cref{alg:greedy}. The most expensive part of this is the computation of $\bA_i\bSv_k\bZ_k$, $i\in \{1,\ldots, l\}$. However, we can speed this up by storing the computations from the previous step. For this, note that $\bSv_{k-1}$ is expanded actively to obtain $\bSv_k$ as follows:
\begin{align*}
    \bSv_{k}= \begin{bmatrix}
    \bSv_{k-1}, & \bs
    \end{bmatrix}.
\end{align*}
For simplicity, we assume that $\bs$ is a vector. Nevertheless, this can also be adapted to matrices. Hence, the error estimation at the $k$th iteration can be computed via
\begin{align*}
  \bR_{k} &= \bA_1 \bSv_{k}\bZ_{k}  \mathbf F_1^\Lambda + \cdots+ \bA_l \bSv_{k}\bZ_{k}  \mathbf F_l^\Lambda - \bB \mathbf{1}^\top\\
  &= \bA_1 \left[\bSv_{k-1}, \bs \right]\bZ_{k}  \mathbf F_1^\Lambda + \cdots+ \bA_l \left[\bSv_{k-1}, \bs \right]\bZ_{k}  \mathbf F_l^\Lambda - \bB \mathbf{1}^\top.
\end{align*}
Note that $\bA_i\bSv_{k-1}$ has been computed in the previous step, which can be pulled from memory. Hence, we only need to calculate $\bA_i\bs$ in order to determine $\bA_i\bSv_{k}$, which can then be stored for the next iteration. The rest of the computations are in low-rank form, which is cheap. 
\subsection{Selecting multiple interpolation points}
\label{subsec:mulpoint}

Despite performing efficient computation for the residual as discussed in the above subsection, evaluating it at each step after adding only the most active point can still be a bottleneck. Hence, to perform this step less frequently, one can select multiple interpolation points in each iteration. To this end,  we take advantage of filter functions to extract more information from the residual $\bR_k$ and the column-wise norms $r_j$ \eqref{eq:norm}. Such an idea has initially been proposed in \cite{morCheGB22}. With a filter function, one aims to exclude the domain of the points  that are already selected for the next iteration through the multiplication of the residual $r_j$ with a suitable filter. In \cite{morCheGB22}, the function 
 \begin{align*}
        \bg(s, \sigma_{\text{sel}})=1-e^{(-\beta(\log(|s|+\epsilon)-\log(|\sigma_{\text{sel}}|+\epsilon))^2)}
    \end{align*}
  is used as a filter function for a single selected interpolation point with the parameters $\beta = 0.6$ and $\epsilon = 10^{-15}$. We make use of the same function in this work. Furthermore, in our numerical experiments, we select three points in every iteration of the active sampling approach in \Cref{alg:greedy}. The first interpolation point  $\sigma_{\text{sel}}^{(1)}$ is selected as usual as in \eqref{eq: sigma}. Next, we apply the filter on the $r_j$, associated with the interpolation point $\sigma_j$ in the training set and replace them with 
    \begin{align*}
        \Tilde{r}_j=\bg(\sigma_j, \sigma_{\text{sel}}^{(1)})r_j, \quad \text{for $j = 1, \dots,  \cN$}.
    \end{align*}  
Then, the second interpolation point $\sigma_{\text{sel}}^{(2)}$ is selected as before but with respect to the updated residual $ \tilde{r}_j$. The procedure is repeated for the third interpolation point in the same way.

Note that we start \Cref{alg:greedy} with a single interpolation (the smallest in magnitude) from the training set to estimate the error at the first iteration. In the following iterations, we consider three points from our training set. Hence, at step $k$, we have $1+3(k-1)$ selected points from the training set.

\subsection{Accelerating the SVD in Algorithm 1 by using low-rank factorization}
In Step 2 of the \drop~ (\Cref{alg:drop}), it is necessary to evaluate the \ac{SVD} of two matrices. We concentrate our discussion on the matrix
\begin{equation}\label{eq:drop_SVD_speedup} 
    \cM =  \left[\bW^{\top}\bA_1 \bV,\dots,\bW^{\top}\bA_l \bV\right] \in \R^{\cN \times l\cN},  
\end{equation}
where $n$ is the order of the high-fidelity model,  $\cN$ is the number of interpolation points in the training set, and $l$ corresponds to the number of matrices defining the \LSS. Hence, if the training set is large, the SVD computation  of $\cM$ could be a major barrier.

In order to efficiently compute the SVD of $\cM$, we utilize the low-rank factorization of the matrices $\bV$ and $\bW$, which are determined using \Cref{alg:greedy}. Let us assume that 
\begin{equation}\label{eq:VW_lowrank}
    \bV \approx \bSv \bZ \quad \text{and}\quad\bW \approx \bSw \bY,
\end{equation}
where $\bSv\in \R^{n \times n_v}$, $\bZ\in \R^{n_v \times \cN}$, $\bSw\in \R^{n \times n_w}$, and $\bY\in \R^{n_w \times \cN}$ with ideally $\{n_v,n_w\} \ll \{n,\cN\}$. 
Thus, we get 
\begin{equation}\label{eq:drop_SVD_speedup_lowrank} 
    \cM \approx \tilde\cM =  \left[\bY^{\top} \left(\bSw\right)^\top\bA_1 \bSv\bZ,\dots,\bY^{\top} \left(\bSw\right)^\top\bA_l \bSv\bZ\right],
\end{equation}
Next, we compute the SVD of the matrices $\bZ$ and $\bY$, denoted by 
\begin{equation}\label{eq:svd_ZY}
\bZ = \bU_z\mathbf{\Sigma}_z\bV_z^{\top} \quad \text{and}\quad \bY = \bU_y\mathbf{\Sigma}_y\bV_y^{\top},
\end{equation}
where $\bU_z \in \mathbb{R}^{n_z \times n_z}$, $\bU_y \in \mathbb{R}^{n_y \times n_y}$, $\bV_z \in \mathbb{R}^{n_z \times \cN}$, are $\bV_y \in \mathbb{R}^{n_y \times \cN}$ are orthonormal matrices, $\mathbf{\Sigma}_z\in \R^{n_z\times n_z}$ and $\mathbf{\Sigma}_y\in \R^{n_y\times n_y}$ are diagonal matrices. The SVD computations in \eqref{eq:svd_ZY} are expected to be much cheaper as they are of much lower dimensions. Substituting these decompositions in  \eqref{eq:drop_SVD_speedup_lowrank}, we obtain
\begin{align}
     \tilde\cM &=  \left[\bV_y\mathbf{\Sigma}_y\bU_y^{\top}  \left(\bSw\right)^\top\bA_1 \bSv\bU_z\mathbf{\Sigma}_z\bV_z^{\top},\dots,\bV_y\mathbf{\Sigma}_y\bU_y^{\top} \left(\bSw\right)^\top\bA_l \bSv\bU_z\mathbf{\Sigma}_z\bV_z^{\top}\right],\\
     & =  \bV_y\underbrace{\left[\mathbf{\Sigma}_y\bU_y^{\top}  \left(\bSw\right)^\top\bA_1 \bSv\bU_z\mathbf{\Sigma}_z,\dots,\mathbf{\Sigma}_y\bU_y^{\top} \left(\bSw\right)^\top\bA_l \bSv\bU_z\mathbf{\Sigma}_z\right]}_{=:\hat\cM}\left(\bI_l \otimes \bV_z^\top\right),
\end{align}
where $\bI_l \in \R^{l\times l}$ is the identity matrix. Leveraging orthonormality of the matrices $\bV_y$ and $\left(\bI_l \otimes \bV_z^\top\right)$, we can compute the SVD of $\hat\cM \in \R^{n_w\times l\cdot n_v}$, which is much smaller in size as compared to the size of $\tilde\cM \in \R^{n\times l\cN}$. Next, we denote the SVD of $\hat\cM$ as follows:
\begin{equation}
    \hat\cM = \bU_m\mathbf{\Sigma}_m\bV_m^\top. 
\end{equation}
As a result, the projection matrix $\bW_{\text{p}}$ in Step 4 of \Cref{alg:drop} can be given as :
\begin{align}
\bW_{\text{p}} &= \bW \bV_y\bU_m(:,1{:}r)\\
& \approx \bSw\bY \bV_y\bU_m(:,1{:}r) &\text{(Substituting for $\bW$ from \eqref{eq:VW_lowrank})}\\
& \approx \bSw\bU_y\mathbf{\Sigma}_y\bV_y^{\top} \bV_y\bU_m(:,1{:}r) &\text{(Substituting for $\bY$ from \eqref{eq:svd_ZY})}\\
&\approx \bSw\bU_y\mathbf{\Sigma}_y\bU_m(:,1{:}r).
\end{align}
Similarly, for the other SVD in Step 2 of the \drop~ procedure (\Cref{alg:drop}), it can be shown that 
\begin{align}
\bV_{\text{p}} & \approx \bSv\bU_v\mathbf{\Sigma}_v\tilde\bV_m(:,1{:}r),
\end{align}
where 
\begin{equation}
     \tilde\bU_m\mathbf{\tilde\Sigma}_m\tilde\bV_m^\top = \begin{bmatrix}\mathbf{\Sigma}_y\bU_y^{\top}  \left(\bSw\right)^\top\bA_1 \bSv\bU_z\mathbf{\Sigma}_z\\\vdots\\ \mathbf{\Sigma}_y\bU_y^{\top} \left(\bSw\right)^\top\bA_l \bSv\bU_z\mathbf{\Sigma}_z\end{bmatrix}. 
\end{equation}
This accelerates the computation to identify the dominant subspace by using SVDs of appropriate matrices, which are all done in a low dimension. More precisely, the complexity of SVDs and computing dominant subspaces depends only on the number of selected interpolation points from the training set and becomes independent of the order of the high-fidelity model and the size of the training set.

\section{Numerical Experiments}\label{sec:Numerics}
In this section, we investigate the performance of the proposed approach with active sampling, denoted \gdrop~, to determine the dominant subspaces for model reduction by actively selecting interpolation points from the training set. We compare \gdrop~~with \drop~ as proposed in \cite{benner2019identification}, in which interpolation points need to be pre-defined and all the points in the training set are taken into account. We present the performance of both approaches on various standard benchmarks. For \gdrop~, we consider a training set by taking equidistant points in a logarithmic scale in a pre-defined frequency range. This range and the number of points are individually stated in each example. We have performed all the computations using  \matlab~R2022a  on an \intel\coreifive-7200U processor at $2.50\,$GHz with 8\,GB RAM. Furthermore, to quantify the performance of \gdrop~~and \drop~, we measure a quantity, which is a function of the transfer functions of the high-fidelity model (denoted by $\bH(s)$) and learned \ac{ROM} (denoted by $\hat\bH(s)$). This is given by
 \begin{equation}
 \label{eq:error}
       \mathbf{e}(s) = {\dfrac{\max_{i}\lambda_i\left(\hat\bH(s)-\bH(s)\right)}{\max_{s \in i\mathbb{R}}{\lambda_1(\bH(s))}}},
\end{equation}
where $\lambda_i$ denotes the $i$th singular values with $\lambda_i \geq \lambda_{i+1}$.

Furthermore, we take the CPU time into account to measure the computational advantage. The CPU time includes, for both \drop~ and \gdrop~, the construction of the matrices $\bV$ and $\bW$ (if two-sided interpolation is considered) to determine reachable and observable subspaces followed by employing \Cref{alg:drop} to identify the dominant reachable and observable subspaces. Finally, to determine the order of the \acp{ROM} (Step 3 in \Cref{alg:drop}), we set $\texttt{tol} = 10^{-8}$. For the construction of $\bV$ in \drop~ we need to solve large-scale linear systems of equations for all interpolation points as in \eqref{eq:intpol}. For \gdrop~~the active sampling procedure as in \Cref{alg:greedy} with $\texttt{tol} = 10^{-3}$ is executed to construct the matrices $\bV$ and $\bW$.
%
\subsection{An LTI model}\label{sec:LTI}
We begin by considering an artificial \emph{full-order} LTI model, which has been widely used as a standard benchmark model \cite{slicot_fom, morPen06}. The model is of order $n = 1~006$ and has the following structure:
\begin{equation}\label{eq:LTI_model}
\begin{aligned}
    \dot{\bx}(t) &= \bA\bx(t) + \bB \bu(t),\\
    \by(t) &= \bC\bx(t),
\end{aligned}
\end{equation}
where $
    \bA \in \mathbb{R}^{n \times n}$, $\bB \in \mathbb{R}^{n \times 1}$,  $\bC \in \mathbb{R}^{1 \times n}$,
and its transfer function can be given by
\begin{align*}
    \bH(s) = \bC(s\bI-\bA)^{-1}\bB,
\end{align*}
with $\bI \in \mathbb{R}^{n \times n}$ being the identity matrix, $s = \jmath\omega$ and $\omega\in \R$. For this example, we assume the interesting frequency range to be $\omega \in [10^{-1}, 10^3]$.\\

\paragraph{Results:} Toward obtaining \acp{ROM}, we first consider $\cN = n= 1~006$ points in our training set in the defined frequency range. 
To employ \drop~, we need to consider all the points in our training set and compute the matrices in \eqref{eq:intpol_main} encoding the reachable and observable subspaces. \gdrop~, on the other hand, selects interpolation points actively based on the residual measure discussed in \Cref{sec:active_sampling}. As a result, we only select $13$ important points out of $1~006$ from our training set, followed by determining dominant subspaces using step 2 in \Cref{alg:drop}. The singular values, which are plotted in \Cref{fig:LTI_dec}, indicate an appropriate order of the \acp{ROM}. 
We notice the singular values from \gdrop~ and \drop~ are the same in the eyeball norm, which is expected since \gdrop~ encodes approximately the same subspace information but with the advantage of requiring fewer computational resources. Next, we determine \acp{ROM} of order $r = 20$  using both methods. We display the transfer functions of the \acp{ROM} and their comparison with the high-fidelity model in \Cref{fig:LTI_TF}. The plots show that the \acp{ROM} faithfully capture the behavior of the high-fidelity model. Moreover, the plot also shows the chosen interpolation points using \gdrop~, which nicely illustrates that \gdrop~ focuses on selecting interpolation points where important dynamics reside. For instance, \gdrop~ automatically selects interpolation points at the peaks, which are essential to capture the dynamics. 

To further illustrate the computational advantage of \gdrop~ over \drop~, we compare the computational costs of both approaches. For this, we vary the number of interpolation points in our training set and  consider different numbers $\left\{n, \frac{n}{2}, \floor*{\frac{n}{5}}\right\}$. We bar-plot the computational costs by noting the CPU time in \Cref{fig:LTI_time}. The figure shows that the cost for \drop~ strongly depends on the number of points in the training set. This is due to the involvement of the large-scale linear systems solves---as many as the size of the training set---and SVDs, whose size depends on how many linear solves are needed. The cost of \gdrop~ does not increase as rapidly as for \drop~ when the size of the training set is increased. For the largest training set, \gdrop~~is more than $12$x faster compared to \drop~. We further highlight that when only few points are selected, this also reduces the SVD computation costs to determine jointly reachable and observable dominant subspaces. This also contributes to the reduction in the CPU time for \gdrop~. However, we note that there is an additional cost for \gdrop~, which is associated with the computation of the residual \eqref{eq:residual}. When the training set is small or/and the order of the \ac{ROM} is large, the latter cost may dominate, thus could make \gdrop~~perform slightly worse. 

\begin{figure}[tb]
\begin{subfigure}[t]{0.6\textwidth}
    \includegraphics[width = 0.965\textwidth]{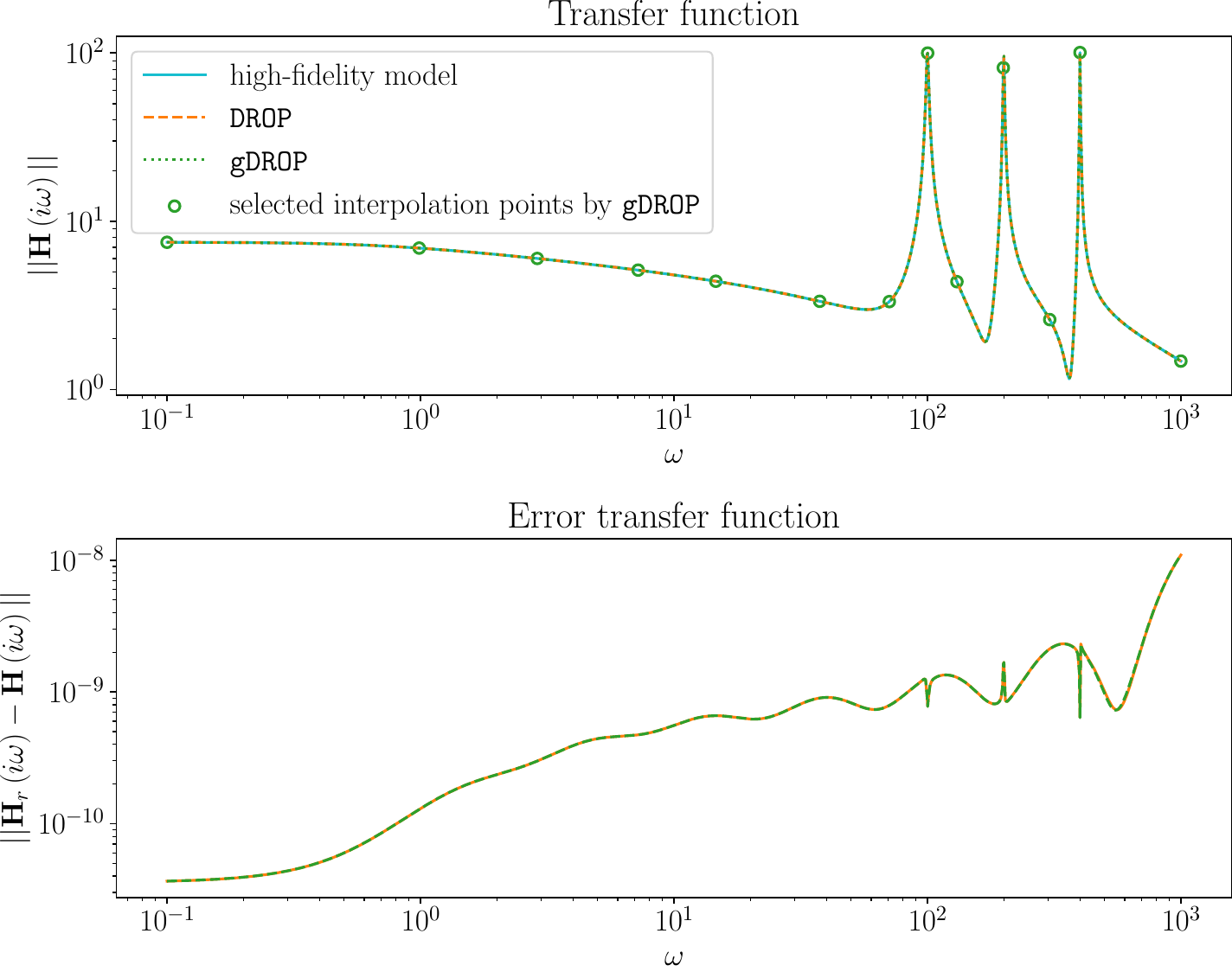}
    \caption{A comparison of the transfer functions for the high-fidelity and \acp{ROM} of order $r=20$.}
    \label{fig:LTI_TF}
\end{subfigure}
\begin{subfigure}[t]{0.395\textwidth}
    \includegraphics[width = 0.945\textwidth]{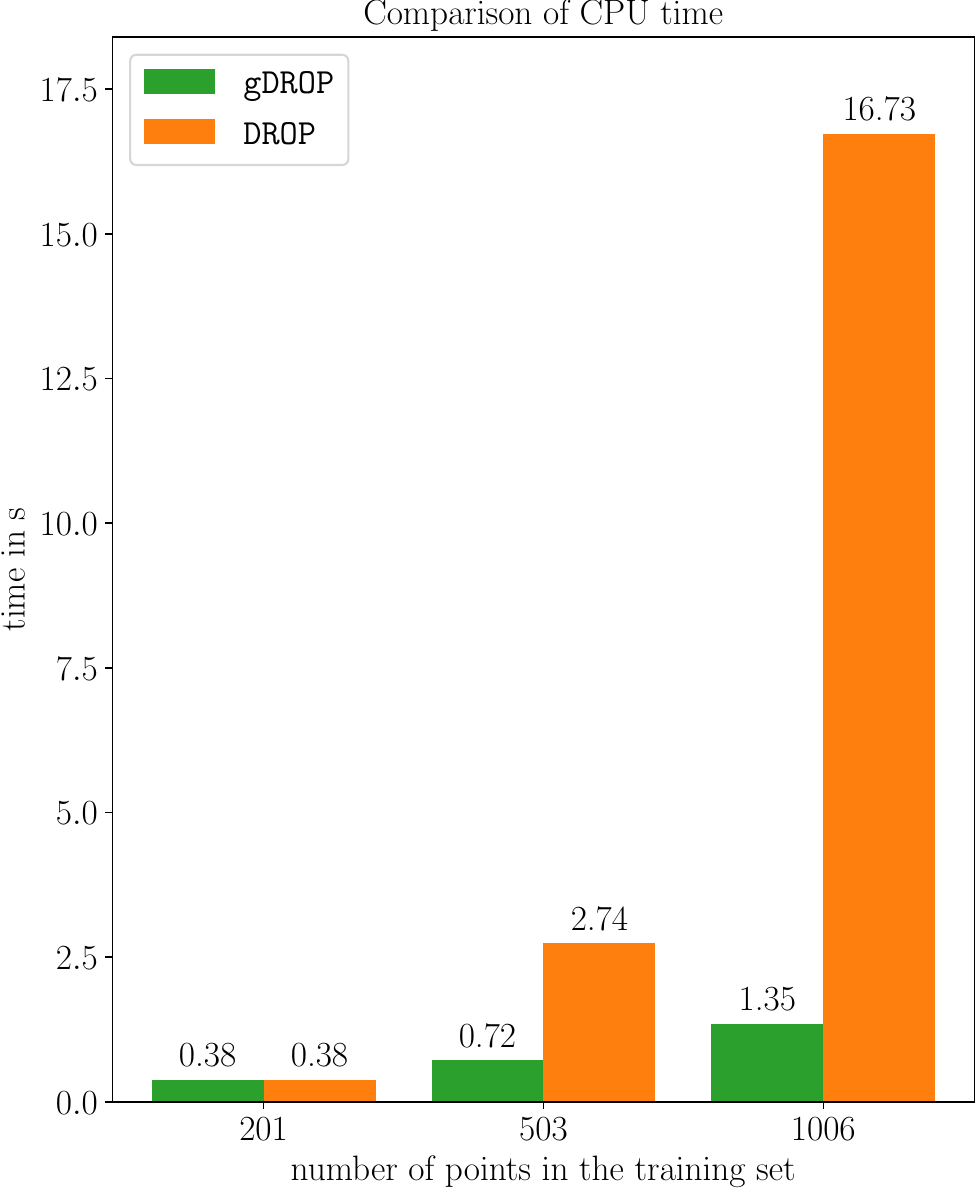}
    \caption{A comparison of computational costs (CPU time) to construct the \acp{ROM}.}
    \label{fig:LTI_time}
\end{subfigure}
\caption{LTI system: A comparison of the performance of \gdrop~~and \drop~.}
\end{figure}

\begin{figure}[tb]
   \centering
     \includegraphics[width = 0.5\textwidth]{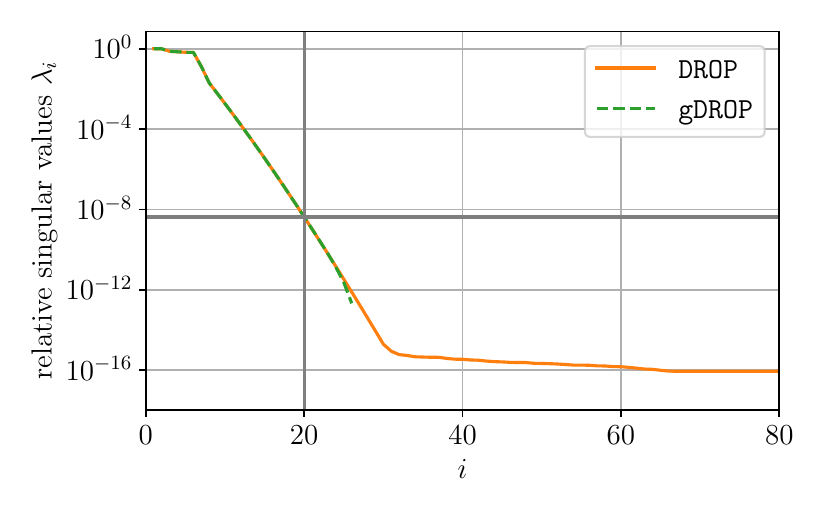}
     \caption{LTI system: relative decay of the singular values for \drop~ and \gdrop~.}
     \label{fig:LTI_dec}
\end{figure}

\subsection{Heating rod model with time-delay}
In the second example, we discuss a SISO time-delay system, initially discussed in \cite{michiels2011krylov}.  This model represents the behavior of a heated rod with distributed control and homogeneous Dirichlet boundary conditions, which is cooled by delayed feedback. The dynamics of the system are governed by a delay differential equation of the following form:
\begin{equation}\label{eq:delay_system}
\begin{aligned}
     \dot{\*x}(t) &= \bA\bx(t) + \*A_\tau\bx(t-\tau) + \*B\*u(t),\\
    \by(t) &= \*C\*x(t),
\end{aligned}
\end{equation}
where $\*A \in \mathbb{R}^{n \times n}$, $\*A_\tau \in \mathbb{R}^{n \times n}$, $\*B \in \mathbb{R}^{n \times 1}$ and $\*C \in \mathbb{R}^{1 \times n}$, $n = 1200$ is the order, and the time constant $\tau = 3$. The transfer function of the system \eqref{eq:delay_system} is given by
\begin{align*}
    \bH(s) = \*C(s\*I-\*A-\*A_\tau e^{-\tau s})^{-1}\*B.
\end{align*}
For this example, we consider the frequency range $\omega \in [10^{-3}, 10^3]$ (recall $s = \jmath \omega)$.
\paragraph{Results:} As in the previous example, we take $1~200$ points for our training set of interpolation points. We then employ \gdrop~ and \drop~ to  first obtain the relevant subspace information. To begin with the reachable subspace, \drop~~considers all $1~200$ points, while \gdrop~~chooses the points actively, thus requiring only seven interpolation points ($14$ if conjugate points are accounted) from the training set. The same holds for the observable subspace. As a result, the computational cost is significantly lower for \gdrop~ compared to \drop~. Based on the singular values (Step 3 in \Cref{alg:drop}), we determine dominant reachable and observable subspaces jointly. We consider $r=8$ to construct \acp{ROM} using a Petrov-Galerkin projection. Next, we compare the \acp{ROM} in \Cref{fig:delay_TF}, where we observe a similar accuracy, which is to be expected since they both have similar subspace information. 

Once again, the advantage of \gdrop~ lies in the computational speed-up. Next, we note the required computational efforts to construct \acp{ROM} using \gdrop~~and \drop~. For this purpose, we consider different scenarios by varying the number of points in the training set; we take  $\left\{n, \frac{n}{2}, \floor*{\frac{n}{5}}\right\}$, where $n$ is the order of the system. We plot the required CPU time for both approaches in \Cref{fig:delay_time}. We make a similar observation as in the previous example; the CPU time for \gdrop~~is significantly lower compared to \drop~ as the number of points in the training set increases. When the number of points is $n$, \gdrop~~is approximately $17$ times faster than \drop~.   

\begin{figure}[tb]
\begin{subfigure}[t]{0.6\textwidth}
    \includegraphics[width = 0.965\textwidth]{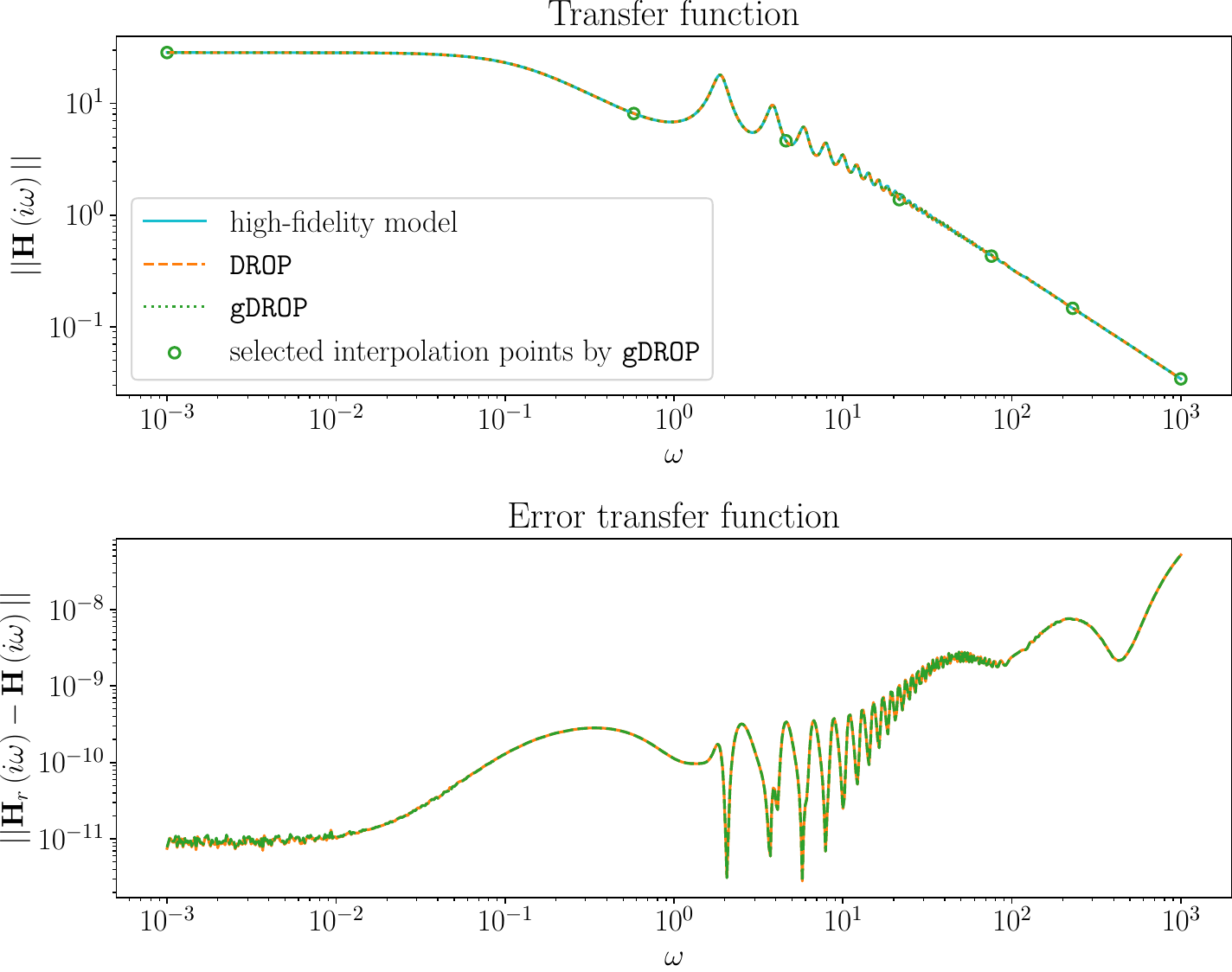}
    \caption{A comparison of the transfer functions of the high-fidelity and \acp{ROM} of order $r=8$.}
    \label{fig:delay_TF}
\end{subfigure}
\begin{subfigure}[t]{0.395\textwidth}
    \includegraphics[width = 0.945\textwidth]{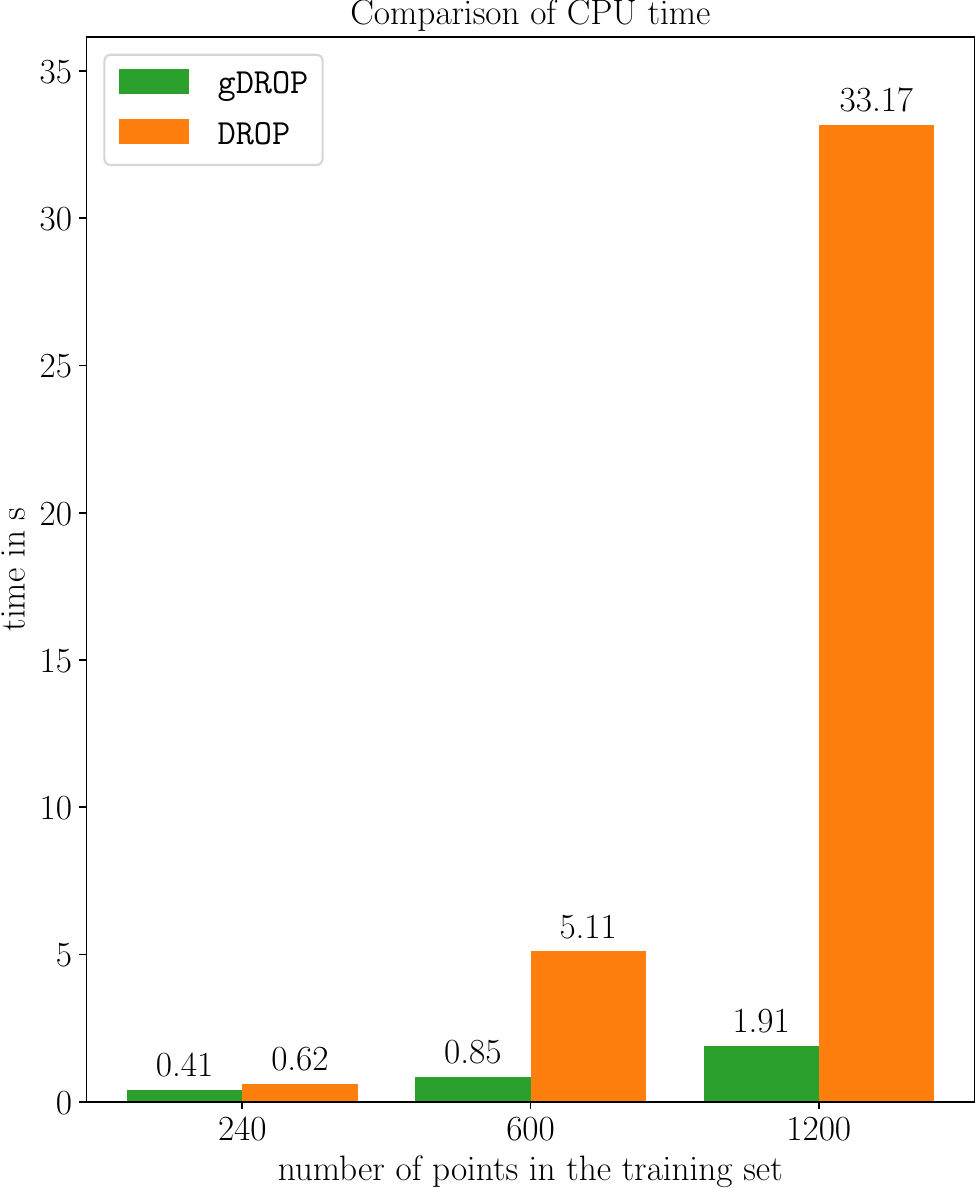}
    \caption{A comparison of the computational costs (CPU time) to construct \acp{ROM}.}
    \label{fig:delay_time}
\end{subfigure}
\caption{Heating rod model: A comparison of the performance of \gdrop~~and \drop~.}
\end{figure}

\subsection{Fading memory heat equation}
For our next example we consider a heat equation with fading memory, which is an integro-differential system of order $n = 16~384$ taken from~\cite{morBre16}. The transfer function is given by
\begin{align*}
\bH(s) = \*C\left(s\*I-\*A+\frac{1}{s+\gamma}\*A\right)^{-1}\*B, 
\end{align*}
where $s = \jmath\omega$ with $\omega \in \R $, $\gamma = 1.05$, $\*A \in \mathbb{R}^{n \times n}$, $\*I \in \mathbb{R}^{n \times n}$ is the identity matrix, $\*B \in \mathbb{R}^{n \times 1}$, and $\*C \in \mathbb{R}^{1 \times n}$. For this example, we consider the range of frequency $\omega \in [10^{-2},10^4]$.
\paragraph{Results:} We perform similar experiments as in the previous two examples. We consider $100$ points in our training set and then employ \gdrop~ and \drop~ to determine reachable and observable subspaces. For \drop~, we are required solve large-scale linear systems for all $100$  points in our training set, whereas \gdrop~ actively selects only $10$ points to approximately obtain the same subspace information. Next, we construct \acp{ROM} of order $r=7$ using both approaches. Their comparison with the high-fidelity models is shown in \Cref{fig:heat_TF}, which is expected to have similar performance, but \gdrop~ has a computational advantage. To stress this point more, we investigate the computational costs for both these methods by considering different numbers of points ($\cN = \{25, 50, 100, 200\}$) in our training set. The result is reported in \Cref{fig:heat_time}. It clearly shows the computational advantage of \gdrop~, since it requires fewer large-scale linear solves. More precisely, when $\cN = 200$, \gdrop~ is over $10$x faster compared to \drop~.

\begin{figure}[tb]
\begin{subfigure}[t]{0.6\textwidth}
    \includegraphics[width = 0.965\textwidth]{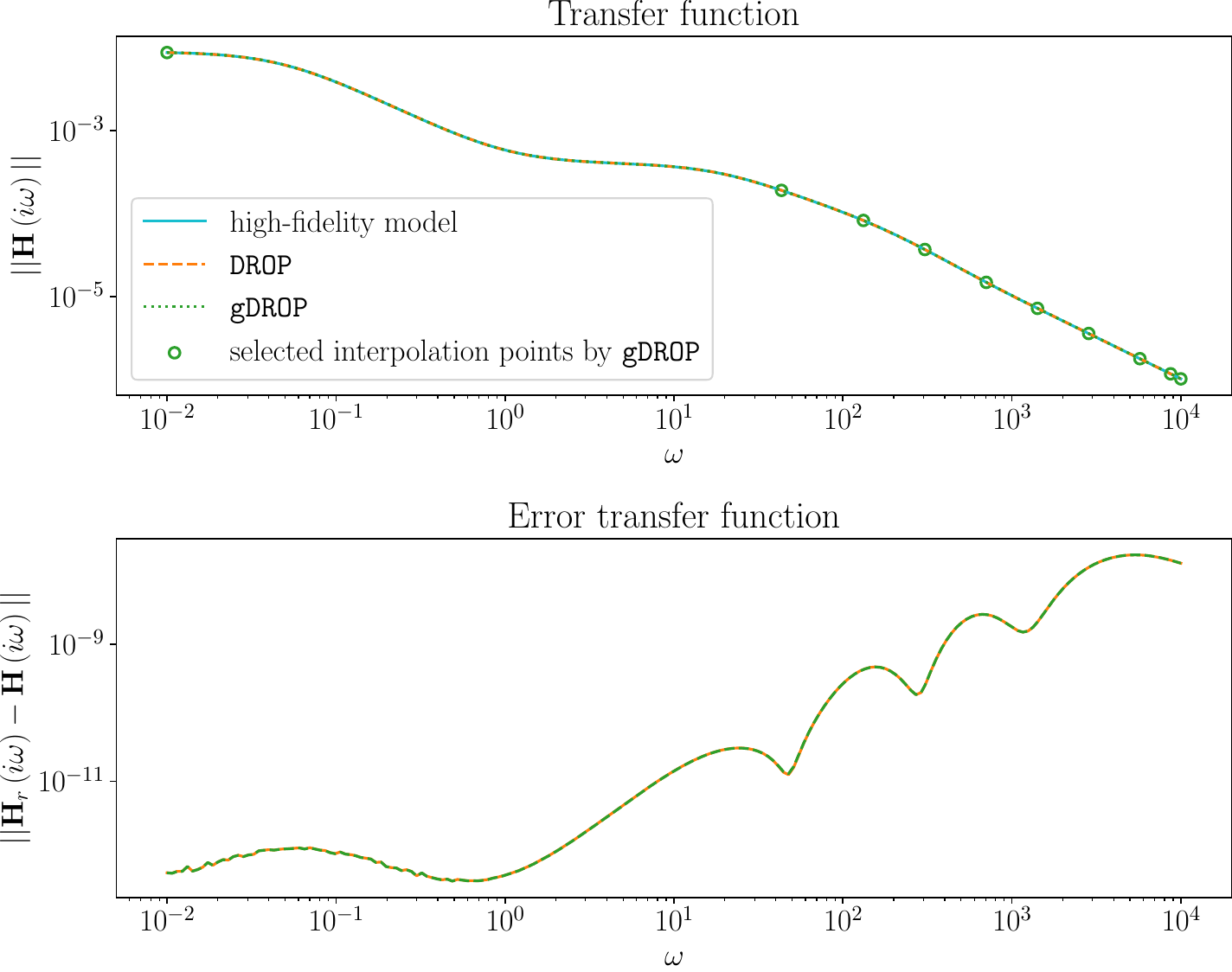}
    \caption{A comparison of the transfer functions of the high-fidelity and \acp{ROM} of order $r=7$.}
    \label{fig:heat_TF}
\end{subfigure}
\begin{subfigure}[t]{0.395\textwidth}
    \includegraphics[width = 0.945\textwidth]{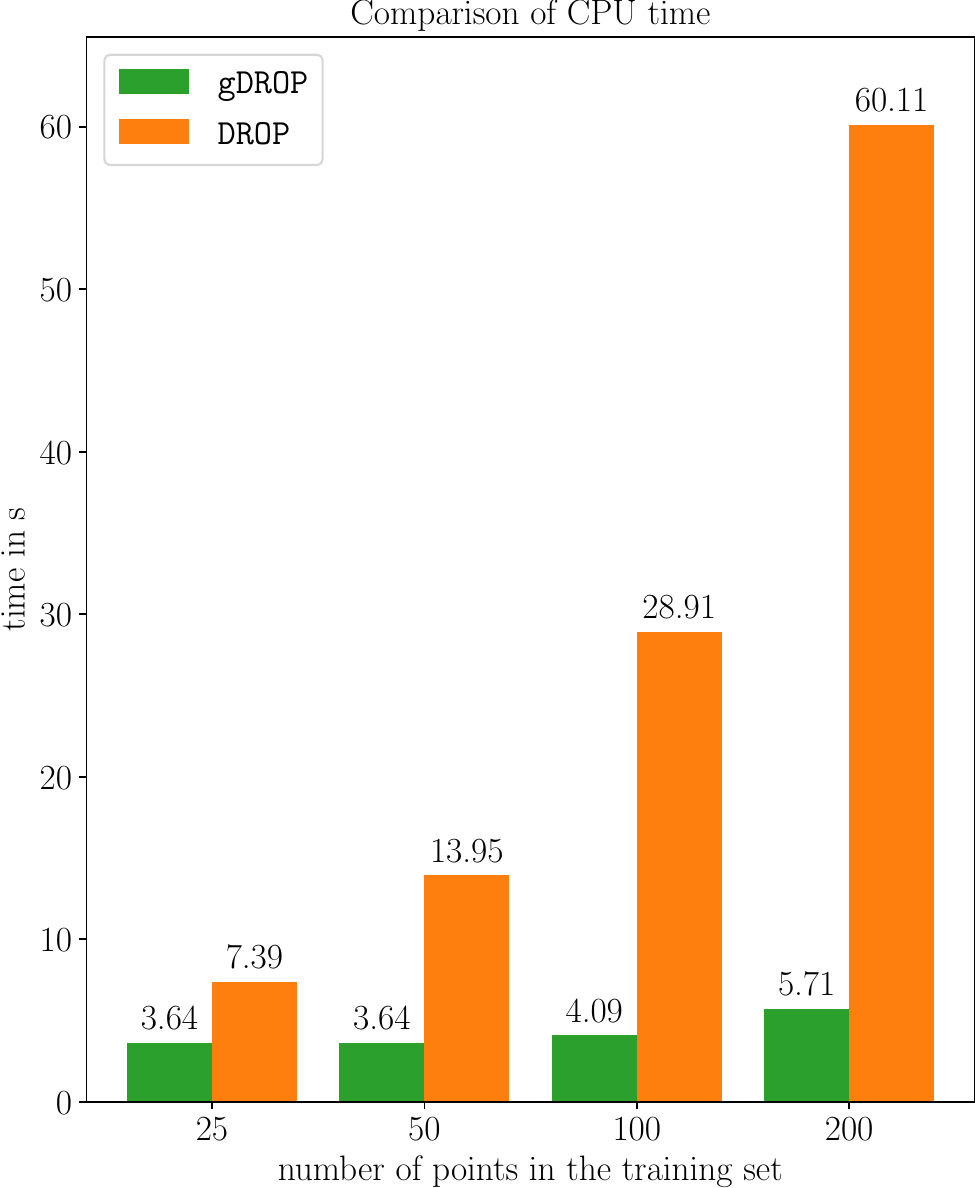}
    \caption{A comparison of the computational costs (CPU time) to construct \acp{ROM}.}
    \label{fig:heat_time}
\end{subfigure}
\caption{Heat Equation: A comparison of the performance of \gdrop~~and \drop~.}
\end{figure}

\subsection{Butterfly Gyroscope}\label{subsec:BF}
The last high-fidelity benchmark model we consider is the butterfly gyroscope model. It is a second-order system of order $ n = 17~361$ with one input and 12 outputs. For further information, we refer to  \cite{morwiki_gyro}. The model is represented by
\begin{equation}\label{eq:BG_equation}
\begin{aligned}
  \*M\ddot{\*x}(t) + \*E\dot{\*x}(t) + \*K\*x(t) &= \*B\*u(t),\\
    \*y(t) &= \*C\*x(t),  
\end{aligned}
\end{equation}
where $\*M, \*E,\, \*K \in \mathbb{R}^{n \times n}$, $\*B \in \mathbb{R}^{n\times 1}$, and $\*C \in \mathbb{R}^{12 \times n}$. Note that $\*M$ is the mass matrix, $\*K$ is the stiffness matrix, and  the damping matrix  $\*E$ is modeled as $\*E=\beta \*K$ with $\beta = 10^{-6}$. The transfer function of the system \eqref{eq:BG_equation} is given by
\begin{equation*}
    \*H(s) = \*C(s^2\*M+s\*E+\*K)^{-1}\*B, 
\end{equation*}
where $s = \jmath\omega$, and $\omega \in \left[10^4,10^6\right]$ is the considered frequency range.
\paragraph{Results:} Second-order models often have special structures In this example, the matrices $\bM, \bK$, and $\bE$ are symmetric positive definite. Therefore, we are interested in preserving these properties in the \acp{ROM}. This can be achieved when \acp{ROM} are constructed using the Galerkin projection of \eqref{eq:BG_equation}; hence, we apply \drop~ and \gdrop~  focusing only on the reachable subspace encoded by the matrix $\bV$ in \eqref{eq:intpol}. Then, we set $\bW = \bV$. The rest of the procedure remains similar to the previous examples. We consider $\cN = 100$ points in our training set, followed by first employing \drop~ and \gdrop~ to determine the reachable subspace. For \drop~, we use all $100$ points from the training set, whereas \gdrop~ actively selects only  $13$ points, which are sufficient to recover the same information. Next, we construct \acp{ROM} of order $r = 18$ using a Galerkin projection, which are expected to provide similar performance. This is precisely what we observe when they are compared; see \Cref{fig:butterfly_TF}. 

Once again, \gdrop~ becomes favorable when the computational cost is taken into account. Therefore, we report the CPU time for both \gdrop~ and \drop~, when the number of points in our training set is varied. For  points $\cN = \{25, 50, 100, 200\}$, we indicate the CPU time in \Cref{fig:butterfly_time} using a bar-plot. We again notice that \drop~ becomes quite expensive as the number of points increases, particularly when the model is of very high fidelity, which is the case for this example. On the other hand, the cost for \gdrop~~increases only moderately with the number of points. Specifically, for $\cN = 200$, \gdrop~~is approximately $12$x faster, showing the superior performance of \gdrop~ to construct \acp{ROM} for very high-fidelity models efficiently.

\begin{figure}[tb]
\begin{subfigure}[t]{0.6\textwidth}
    \includegraphics[width = 0.965\textwidth]{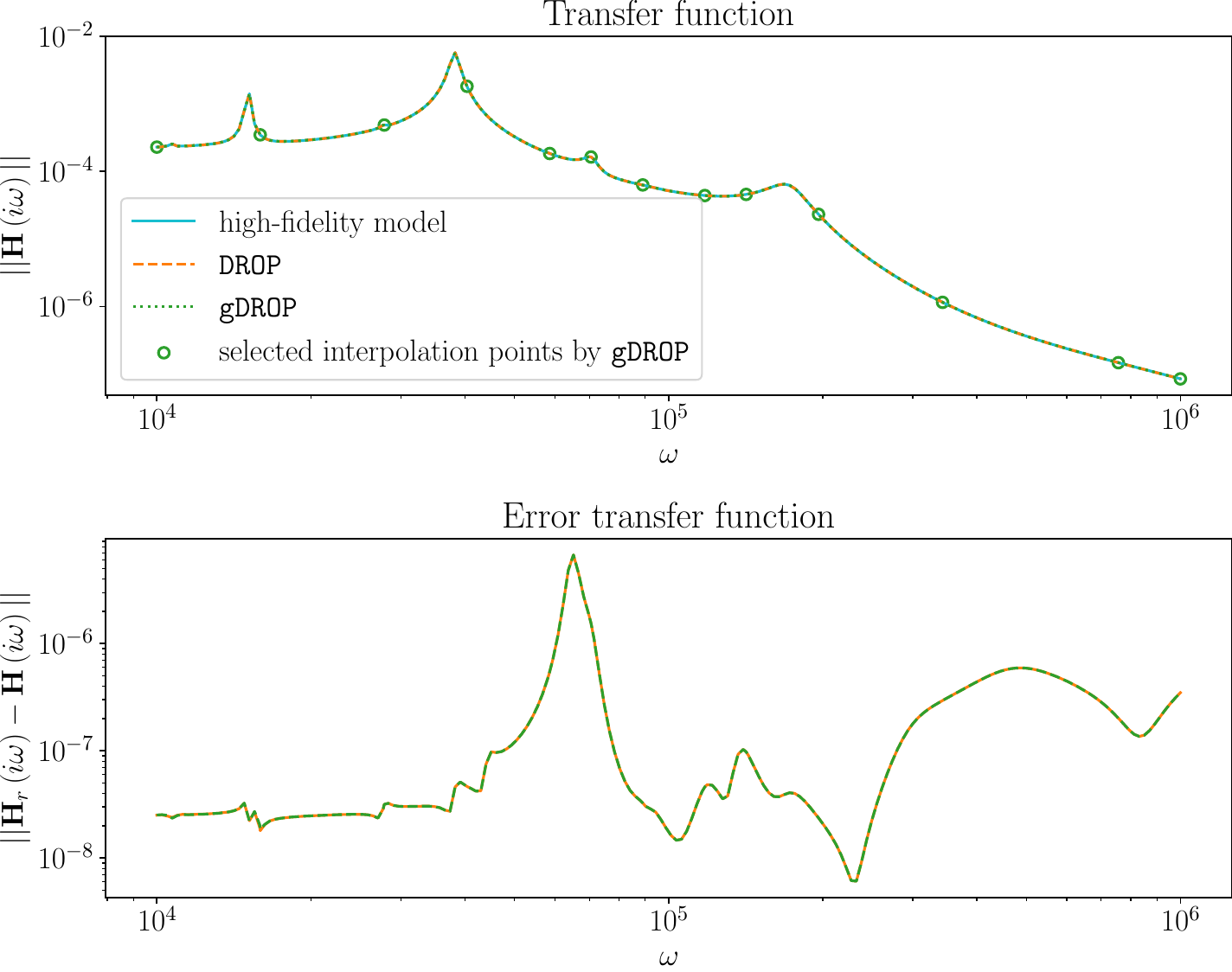}
    \caption{A comparison of the transfer functions of the high-fidelity and \acp{ROM} of order $r=18$.}
    \label{fig:butterfly_TF}
\end{subfigure}
\begin{subfigure}[t]{0.395\textwidth}
    \includegraphics[width = 0.945\textwidth]{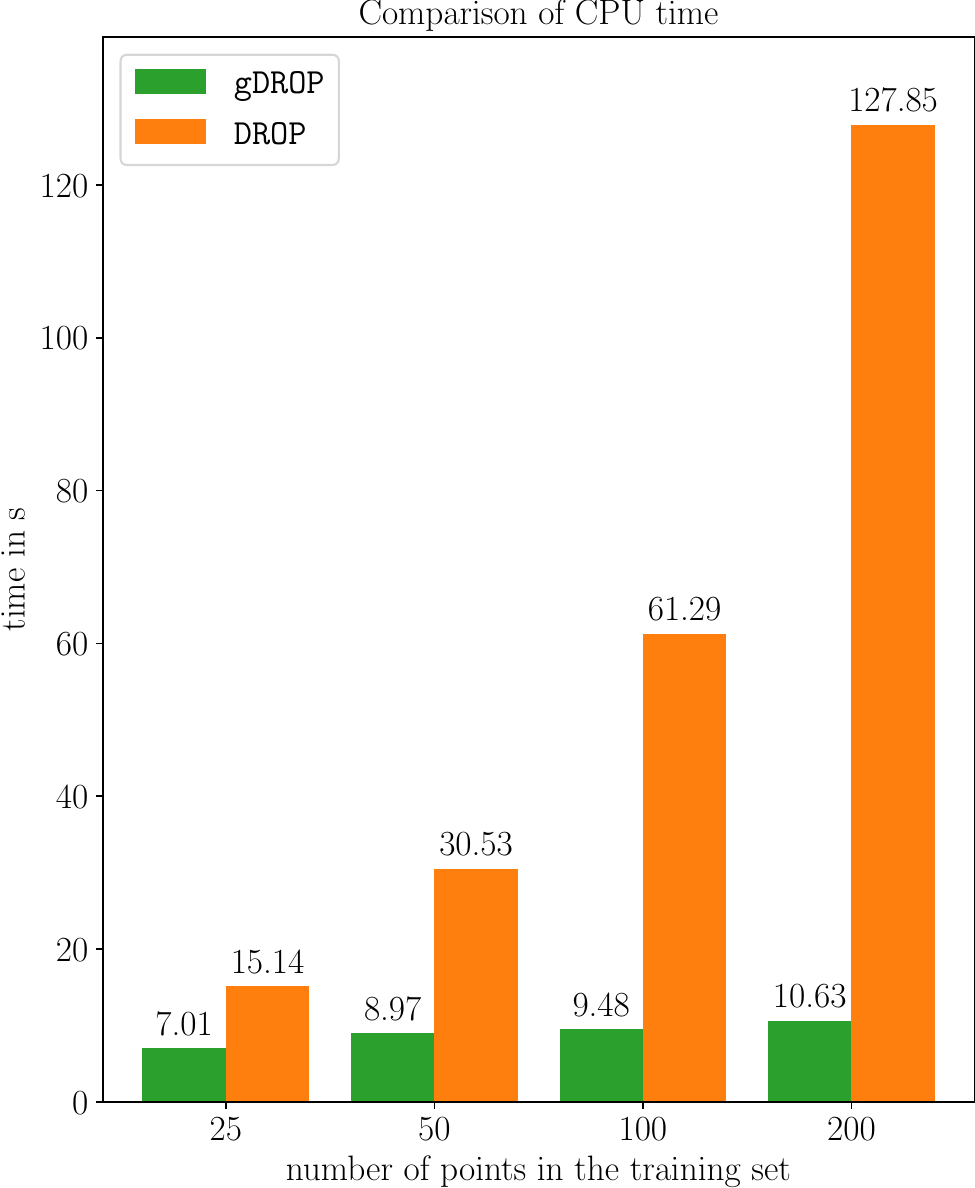}
    \caption{A comparison of the computational costs (CPU time) to construct \acp{ROM}.}
    \label{fig:butterfly_time}
\end{subfigure}
\caption{Butterfly Gyroscope: A comparison of the performance of \gdrop~ and \drop~.}
\end{figure}


\section*{Conclusions}
We have proposed an active sampling procedure enabling efficient computation of the dominant reachable and observable subspaces for linear structured systems, allowing us to compute reduced-order surrogate models. For a large training set of interpolation points, reachable and observable subspaces can be determined by solving large-scale linear solves, which typically require as many as the number of interpolation points in the training set. This imposes a computational burden, especially in the case of high-fidelity systems. Towards developing an active sampling strategy,  we have first cast the identification of the subspaces as a solution to generalized Sylvester equations, exhibiting an interpolatory structure. By utilizing this particular connection between interpolation points and the columns of the solution of the matrix equation, we determine the active points that most likely  provide us with relevant information about the desired subspaces. We actively sample points from the training set in an iterative manner and obtain solutions of the matrix equations in low-rank forms, encoding the subspaces accurately. Particular attention was paid to the computational aspects, and we have shown how the low-rank form of the solutions can significantly speed up the process of obtaining reduced-order models. 
In addition, we have demonstrated the efficiency of the active sampling strategy to determine a few important points from the training set to obtain subspace information. 
There, we observed a speed-up of more than factor $12$ using the proposed active sampling strategy compared to the approach where a uniform sampling of interpolation points is chosen. In our future work, we would like to use the active sampling scheme for parametric and nonlinear systems \cite{morBenG21,morGoyPB23}. For this, one can potentially combine the presented active sampling strategy with the idea of radial basis proposed, e.g. in \cite{morCheFB22}.

\section*{Data Availability}%
A repository containing an implementation of \Cref{alg:greedy} as well as the code to reproduce the presented results and figures can be found on Zenodo \cite{reddig_2024_11980821}.
\section*{Acknowledgments}%
\addcontentsline{toc}{section}{Acknowledgments}
We would like to express our gratitude to Dr.\ Sridhar Chellappa for several fruitful discussions on this work.


\addcontentsline{toc}{section}{References}
\bibliographystyle{siamplain}
\bibliography{mor}

\begin{thebibliography}{10}

\bibitem{morAntBG20}
{\sc A.~C. Antoulas, C.~A. Beattie, and S.~Gugercin}, {\em Interpolatory Methods for Model Reduction}, Computational Science \& Engineering, Society for Industrial and Applied Mathematics, Philadelphia, PA, 2020, \href{http://dx.doi.org/10.1137/1.9781611976083}{doi:\nolinkurl{10.1137/1.9781611976083}}.

\bibitem{morAntGI16}
{\sc A.~C. Antoulas, I.~V. Gosea, and A.~C. Ionita}, {\em Model reduction of bilinear systems in the {L}oewner framework}, {SIAM} J. Sci. Comput., 38 (2016), pp.~B889--B916, \href{http://dx.doi.org/10.1137/15M1041432}{doi:\nolinkurl{10.1137/15M1041432}}.

\bibitem{morBeaB14}
{\sc C.~A. Beattie and P.~Benner}, {\em {$\mathcal{H}_2$}-optimality conditions for structured dynamical systems}, Preprint MPIMD/14-18, Max Planck Institute Magdeburg, 2014, \url{https://csc.mpi-magdeburg.mpg.de/preprints/2014/18/}.

\bibitem{morBeaG09}
{\sc C.~A. Beattie and S.~Gugercin}, {\em Interpolatory projection methods for structure-preserving model reduction}, Systems Control Lett., 58 (2009), pp.~225--232, \href{http://dx.doi.org/10.1016/j.sysconle.2008.10.016}{doi:\nolinkurl{10.1016/j.sysconle.2008.10.016}}.

\bibitem{morBenG21}
{\sc P.~Benner and P.~Goyal}, {\em Interpolation-based model order reduction for polynomial systems}, {SIAM} J. Sci. Comput., 43 (2021), pp.~A84--A108, \href{http://dx.doi.org/10.1137/19M1259171}{doi:\nolinkurl{10.1137/19M1259171}}.

\bibitem{benner2019identification}
{\sc P.~Benner, P.~Goyal, and I.~Pontes~Duff}, {\em Identification of dominant subspaces for model reduction of structured parametric systems}, Internat. J. Numer. Methods Engrg.,  (2024), p.~e7496, \href{http://dx.doi.org/10.1002/nme.7496}{doi:\nolinkurl{10.1002/nme.7496}}.

\bibitem{morBenMS05}
{\sc P.~Benner, V.~Mehrmann, and D.~C. Sorensen}, {\em Dimension Reduction of Large-Scale Systems}, vol.~45 of Lect. Notes Comput. Sci. Eng., Springer-Verlag, Berlin/Heidelberg, Germany, 2005, \href{http://dx.doi.org/10.1007/3-540-27909-1}{doi:\nolinkurl{10.1007/3-540-27909-1}}.

\bibitem{morBenOCetal17}
{\sc P.~Benner, M.~Ohlberger, A.~Cohen, and K.~Willcox}, {\em Model Reduction and Approximation: Theory and Algorithms}, Computational Science \& Engineering, Society for Industrial and Applied Mathematics, Philadelphia, PA, 2017, \href{http://dx.doi.org/10.1137/1.9781611974829}{doi:\nolinkurl{10.1137/1.9781611974829}}.

\bibitem{BenS13}
{\sc P.~Benner and J.~Saak}, {\em Numerical solution of large and sparse continuous time algebraic matrix {R}iccati and {L}yapunov equations: a state of the art survey}, GAMM-Mitt., 36 (2013), pp.~32--52, \href{http://dx.doi.org/10.1002/gamm.201310003}{doi:\nolinkurl{10.1002/gamm.201310003}}.

\bibitem{morBenSGetal21v1}
{\sc P.~Benner, W.~Schilders, S.~Grivet-Talocia, A.~Quarteroni, G.~Rozza, and L.~M. Silveira}, {\em {Model Order Reduction. Volume 1: System- and Data-Driven Methods and Algorithms}}, De~Gruyter, Berlin, 2021, \url{https://www.degruyter.com/view/title/523453}.

\bibitem{morBre16}
{\sc T.~Breiten}, {\em Structure-preserving model reduction for integro-differential equations}, {SIAM} J. Control Optim., 54 (2016), pp.~2992--3015, \href{http://dx.doi.org/10.1137/15M1032296}{doi:\nolinkurl{10.1137/15M1032296}}.

\bibitem{morChaGVetal05}
{\sc V.~Chahlaoui, K.~A. Gallivan, A.~Vandendorpe, and P.~{Van Dooren}}, {\em Model reduction of second-order systems}, in Dimension Reduction of Large-Scale Systems, P.~Benner, V.~Mehrmann, and D.~C. Sorensen, eds., vol.~45 of Lect. Notes Comput. Sci. Eng., Springer-Verlag, Berlin/Heidelberg, Germany, 2005, pp.~149--172, \href{http://dx.doi.org/10.1007/3-540-27909-1_6}{doi:\nolinkurl{10.1007/3-540-27909-1_6}}.

\bibitem{morCheFB22}
{\sc S.~Chellappa, L.~Feng, and P.~Benner}, {\em An adaptive sampling approach for the reduced basis method}, in Realization and Model Reduction of Dynamical Systems - A Festschrift in Honor of the 70th Birthday of {T}hanos {A}ntoulas, Springer, Cham, 2022, pp.~137--155, \href{http://dx.doi.org/10.1007/978-3-030-95157-3_8}{doi:\nolinkurl{10.1007/978-3-030-95157-3_8}}.

\bibitem{morCheFdetal21a}
{\sc S.~Chellappa, L.~Feng, V.~de~la Rubia, and P.~Benner}, {\em Adaptive interpolatory {MOR} by learning the error estimator in the parameter domain}, in Model Reduction of Complex Dynamical Systems, vol.~171 of International Series of Numerical Mathematics, Birkh{\"a}user, Cham, 2021, pp.~97--117, \href{http://dx.doi.org/10.1007/978-3-030-72983-7_5}{doi:\nolinkurl{10.1007/978-3-030-72983-7_5}}.

\bibitem{morCheGB22}
{\sc K.~Cherifi, P.~Goyal, and P.~Benner}, {\em A greedy data collection scheme for linear dynamical systems}, Data-Centric Engineering, 3 (2022), e16, \href{http://dx.doi.org/10.1017/dce.2022.16}{doi:\nolinkurl{10.1017/dce.2022.16}}.

\bibitem{morDroHO12}
{\sc M.~Drohmann, B.~Haasdonk, and M.~Ohlberger}, {\em Reduced basis approximation for nonlinear parametrized evolution equations based on empirical operator interpolation}, {SIAM} J. Sci. Comput., 34 (2012), pp.~A937--A969, \href{http://dx.doi.org/10.1137/10081157X}{doi:\nolinkurl{10.1137/10081157X}}.

\bibitem{morFenAB17}
{\sc L.~Feng, A.~C. Antoulas, and P.~Benner}, {\em Some a posteriori error bounds for reduced order modelling of (non-)parametrized linear systems}, ESAIM: M2AN, 51 (2017), pp.~2127--2158, \href{http://dx.doi.org/10.1051/m2an/2017014}{doi:\nolinkurl{10.1051/m2an/2017014}}.

\bibitem{morFenB19b}
{\sc L.~Feng and P.~Benner}, {\em A new error estimator for reduced-order modeling of linear parametric systems}, {IEEE} Trans. Microw. Theory Techn., 67 (2019), pp.~4848--4859, \href{http://dx.doi.org/10.1109/TMTT.2019.2948858}{doi:\nolinkurl{10.1109/TMTT.2019.2948858}}.

\bibitem{morGalVV04}
{\sc K.~Gallivan, A.~Vandendorpe, and P.~Van~Dooren}, {\em Model reduction of {MIMO} systems via tangential interpolation}, {SIAM} J. Matrix Anal. Appl., 26 (2004), pp.~328--349, \href{http://dx.doi.org/10.1137/S0895479803423925}{doi:\nolinkurl{10.1137/S0895479803423925}}.

\bibitem{morGalVV04a}
{\sc K.~Gallivan, A.~Vandendorpe, and P.~Van~Dooren}, {\em {S}ylvester equations and projection-based model reduction}, J. Comput. Appl. Math., 162 (2004), pp.~213--229, \href{http://dx.doi.org/10.1016/j.cam.2003.08.026}{doi:\nolinkurl{10.1016/j.cam.2003.08.026}}.

\bibitem{morGosA18}
{\sc I.~V. Gosea and A.~C. Antoulas}, {\em Data-driven model order reduction of quadratic-bilinear systems}, Numer. Lin. Alg. Appl., 25 (2018), p.~e2200, \href{http://dx.doi.org/10.1002/nla.2200}{doi:\nolinkurl{10.1002/nla.2200}}.

\bibitem{morGoyPB23}
{\sc P.~Goyal, I.~Pontes~Duff, and P.~Benner}, {\em Dominant subspaces of high-fidelity polynomial structured parametric dynamical systems and model reduction}, Adv. Comput. Math., 50 (2024), p.~42, \href{http://dx.doi.org/10.1007/s10444-024-10133-8}{doi:\nolinkurl{10.1007/s10444-024-10133-8}}.

\bibitem{morGri97}
{\sc E.~J. Grimme}, {\em {K}rylov projection methods for model reduction}, {Ph.D. Thesis}, Univ. of Illinois at Urbana-Champaign, USA, 1997, \url{https://perso.uclouvain.be/paul.vandooren/ThesisGrimme.pdf}.

\bibitem{morGugAB08}
{\sc S.~Gugercin, A.~C. Antoulas, and C.~Beattie}, {\em {$\mathcal{H}_2$} model reduction for large-scale linear dynamical systems}, {SIAM} J. Matrix Anal. Appl., 30 (2008), pp.~609--638, \href{http://dx.doi.org/10.1137/060666123}{doi:\nolinkurl{10.1137/060666123}}.

\bibitem{morHaa17}
{\sc B.~Haasdonk}, {\em {R}educed basis methods for parametrized {PDE}s---{A} tutorial introduction for stationary and instationary problems}, in Model Reduction and Approximation: Theory and Algorithms, P.~Benner, A.~Cohen, M.~Ohlberger, and K.~Willcox, eds., SIAM, 2017, pp.~65--136, \href{http://dx.doi.org/10.1137/1.9781611974829.ch2}{doi:\nolinkurl{10.1137/1.9781611974829.ch2}}.

\bibitem{morHaaO08a}
{\sc B.~Haasdonk, M.~Ohlberger, and G.~Rozza}, {\em {A} reduced basis method for evolution schemes with parameter-dependent explicit operators}, Electron. Trans. Numer. Anal., 32 (2008), pp.~145--168, \href{http://dx.doi.org/10.1051/m2an:2008001}{doi:\nolinkurl{10.1051/m2an:2008001}}.

\bibitem{morJarDM13}
{\sc E.~Jarlebring, T.~Damm, and W.~Michiels}, {\em Model reduction of time-delay systems using position balancing and delay {L}yapunov equations}, Math. Control Signals Systems, 25 (2013), pp.~147--166, \href{http://dx.doi.org/10.1007/s00498-012-0096-9}{doi:\nolinkurl{10.1007/s00498-012-0096-9}}.

\bibitem{kressner2015truncated}
{\sc D.~Kressner and P.~Sirkovi{\'c}}, {\em Truncated low-rank methods for solving general linear matrix equations}, Linear Algebra Appl., 22 (2015), pp.~564--583, \href{http://dx.doi.org/10.1002/nla.1973}{doi:\nolinkurl{10.1002/nla.1973}}.

\bibitem{morLam93}
{\sc J.~Lam}, {\em Model reduction of delay systems using {P}ad{\'e} approximants}, Internat. J. Control, 57 (1993), pp.~377--391, \href{http://dx.doi.org/10.1080/00207179308934394}{doi:\nolinkurl{10.1080/00207179308934394}}.

\bibitem{lazar2021greedy}
{\sc M.~Lazar and J.~Weston}, {\em Greedy algorithm for parameter dependent operator {L}yapunov equations}, Systems Control Lett., 154 (2021), \href{http://dx.doi.org/10.1016/j.sysconle.2021.104968}{doi:\nolinkurl{10.1016/j.sysconle.2021.104968}}.

\bibitem{morMayA07}
{\sc A.~J. Mayo and A.~C. Antoulas}, {\em A framework for the solution of the generalized realization problem}, Linear Algebra Appl., 425 (2007), pp.~634--662, \href{http://dx.doi.org/10.1016/j.laa.2007.03.008}{doi:\nolinkurl{10.1016/j.laa.2007.03.008}}.

\bibitem{michiels2011krylov}
{\sc W.~Michiels, E.~Jarlebring, and K.~Meerbergen}, {\em {K}rylov-based model order reduction of time-delay systems}, {SIAM} J. Matrix Anal. Appl., 32 (2011), pp.~1399--1421, \href{http://dx.doi.org/10.1137/100797436}{doi:\nolinkurl{10.1137/100797436}}.

\bibitem{morMli20}
{\sc P.~Mlinari{\'c}}, {\em Structure-preserving model order reduction for network systems}, {D}issertation, Otto-von-Guericke-Universit{\"a}t, Magdeburg, Germany, 2020, \href{http://dx.doi.org/10.25673/33570}{doi:\nolinkurl{10.25673/33570}}.

\bibitem{morMoo81}
{\sc B.~C. Moore}, {\em Principal component analysis in linear systems: controllability, observability, and model reduction}, {IEEE} Trans. Autom. Control, AC--26 (1981), pp.~17--32, \href{http://dx.doi.org/10.1109/TAC.1981.1102568}{doi:\nolinkurl{10.1109/TAC.1981.1102568}}.

\bibitem{slicot_fom}
{\sc Niconet e.V.}, {\em {SLICOT} - Subroutine Library in Systems and Control Theory}, \url{http://www.slicot.org}.

\bibitem{morwiki_gyro}
{\sc {Oberwolfach Benchmark Collection}}, {\em Butterfly {G}yroscope}.
\newblock hosted at {MORwiki} -- Model Order Reduction Wiki, 2004, \url{http://modelreduction.org/index.php/Butterfly_Gyroscope}.

\bibitem{morPen06}
{\sc T.~Penzl}, {\em Algorithms for model reduction of large dynamical systems}, Linear Algebra Appl., 415 (2006), pp.~322--343, \href{http://dx.doi.org/10.1016/j.laa.2006.01.007}{doi:\nolinkurl{10.1016/j.laa.2006.01.007}}.
\newblock (Reprint of Technical Report SFB393/99-40, TU Chemnitz, 1999.).

\bibitem{morPonGBetal16}
{\sc I.~Pontes~Duff, S.~Gugercin, C.~Beattie, C.~Poussot-Vassal, and C.~Seren}, {\em $\mathcal{H}_{2}$-optimality conditions for reduced time-delay systems of dimensions one}, IFAC-PapersOnLine, 49 (2016), pp.~7--12, \href{http://dx.doi.org/10.1016/j.ifacol.2016.07.464}{doi:\nolinkurl{10.1016/j.ifacol.2016.07.464}}.

\bibitem{morPonPS18}
{\sc I.~Pontes~Duff, C.~Poussot-Vassal, and C.~Seren}, {\em $\mathcal{H}_{2}$-optimal model approximation by input/output-delay structured reduced-order models}, Systems Control Lett., 117 (2018), pp.~60--67, \href{http://dx.doi.org/10.1016/j.sysconle.2018.05.003}{doi:\nolinkurl{10.1016/j.sysconle.2018.05.003}}.

\bibitem{morPrzV21}
{\sc J.~Przybilla and M.~Voigt}, {\em Model reduction of parametric differential-algebraic systems by balanced truncation}, e-print 2108.08646, arXiv, 2021, \url{https://arxiv.org/abs/2108.08646}.
\newblock math.DS.

\bibitem{morQuaMN16}
{\sc A.~Quarteroni, A.~Manzoni, and F.~Negri}, {\em Reduced Basis Methods for Partial Differential Equations}, vol.~92 of La Matematica per il 3+2, Springer International Publishing, 2016.
\newblock ISBN: 978-3-319-15430-5.

\bibitem{reddig_2024_11980821}
{\sc C.~Reddig, P.~Goyal, I.~Pontes~Duff, and P.~Benner}, {\em {Active Sampling of Interpolation Points to Identify Dominant Subspaces for Model Reduction}}, June 2024, \href{http://dx.doi.org/10.5281/zenodo.11980821}{doi:\nolinkurl{10.5281/zenodo.11980821}}, \url{https://doi.org/10.5281/zenodo.11980821}.

\bibitem{morReiS08}
{\sc T.~Reis and T.~Stykel}, {\em Balanced truncation model reduction of second-order systems}, Math. Comput. Model. Dyn. Syst., 14 (2008), pp.~391--406, \href{http://dx.doi.org/10.1080/13873950701844170}{doi:\nolinkurl{10.1080/13873950701844170}}.

\bibitem{morSaaSW19}
{\sc J.~Saak, D.~Siebelts, and S.~W.~R. Werner}, {\em A comparison of second-order model order reduction methods for an artificial fishtail}, at-Auto\-mati\-sie\-rungs\-tech\-nik, 67 (2019), pp.~648--667, \href{http://dx.doi.org/10.1515/auto-2019-0027}{doi:\nolinkurl{10.1515/auto-2019-0027}}.

\bibitem{morSchH15}
{\sc A.~Schmidt and B.~Haasdonk}, {\em Reduced basis approximation of large scale algebraic {R}iccati equations}, tech. report, University of Stuttgart, 2015, \url{http://www.simtech.uni-stuttgart.de/publikationen/prints.php?ID=999}.

\bibitem{morSchUBG18}
{\sc P.~Schulze, B.~Unger, C.~Beattie, and S.~Gugercin}, {\em Data-driven structured realization}, Linear Algebra Appl., 537 (2018), pp.~250--286, \href{http://dx.doi.org/10.1016/j.laa.2017.09.030}{doi:\nolinkurl{10.1016/j.laa.2017.09.030}}.

\bibitem{Sim16}
{\sc V.~Simoncini}, {\em Computational methods for linear matrix equations}, {SIAM} Rev., 58 (2016), pp.~377--441, \href{http://dx.doi.org/10.1137/130912839}{doi:\nolinkurl{10.1137/130912839}}.

\bibitem{morSonS17}
{\sc N.~T. Son and T.~Stykel}, {\em Solving parameter-dependent {L}yapunov equations using the reduced basis method with application to parametric model order reduction}, {SIAM} J. Matrix Anal. Appl., 38 (2017), pp.~478--504, \href{http://dx.doi.org/10.1137/15M1027097}{doi:\nolinkurl{10.1137/15M1027097}}.

\end{thebibliography}
  
\end{document}